\documentclass{article}

\PassOptionsToPackage{numbers,compress}{natbib}
\usepackage[final, main]{neurips_2026}

\usepackage[utf8]{inputenc}
\usepackage[T1]{fontenc}
\usepackage{hyperref}
\hypersetup{hidelinks,hypertexnames=false,pdfauthor={},pdftitle={},pdfsubject={},pdfkeywords={}}
\usepackage{url}
\usepackage{booktabs}
\usepackage{amsfonts}
\usepackage{amssymb}
\usepackage{amsmath}
\usepackage{mathtools}
\usepackage{amsthm}
\usepackage{dsfont}
\usepackage{nicefrac}
\usepackage{microtype}
\usepackage[table,dvipsnames]{xcolor}
\usepackage{multirow}
\usepackage{algorithm}
\usepackage{algpseudocode}
\usepackage{graphicx}
\usepackage{subcaption}
\usepackage{enumitem}
\usepackage{array}
\usepackage{tikz}
\usepackage{mdframed}
\usepackage{wrapfig}
\usepackage{pifont}
\usepackage{placeins}

\newcommand{\cmark}{\textcolor{ForestGreen}{\ding{51}}}
\newcommand{\xmark}{\textcolor{red}{\ding{55}}}
\newcommand{\ours}{StruPO}

\definecolor{mygray}{gray}{.93}

\theoremstyle{plain}
\newtheorem{theorem}{Theorem}[section]

\theoremstyle{definition}

\theoremstyle{remark}

\title{From Structural Feedback to Prompt Policies: Learning Faithful Text-to-Image Prompt Editors}
\author{%
  \textbf{%
    Wentao Ye$^{1}$\thanks{Equal contribution.} \quad
    Yali Ye$^{2}$\footnotemark[1] \quad
    Zhiqing Xiao$^{1}$\footnotemark[1] \quad
    Ru Peng$^{1}$ \quad
    Xinpeng Ti$^{1}$%
  }\\[0.4em]
  \textbf{%
    Liyao Li$^{1}$ \quad
    Zhanming Shen$^{1}$ \quad
    Peng Lu$^{1}$\thanks{Corresponding authors: Peng Lu and Jie Zhang.} \quad
    Jie Zhang$^{3}$\footnotemark[2] \quad
    Haobo Wang$^{1}$%
  }\\[0.8em]
  {\normalfont\normalsize
    $^{1}$Zhejiang University
    \qquad
    $^{2}$Tencent AI Lab
  }\\[0.3em]
  {\normalfont\normalsize
    $^{3}$China Telecom Artificial Intelligence Technology (Beijing) Co. Ltd
  }%
}

\begin{document}
\maketitle

\begin{abstract}
Prompt optimization should improve the image a user receives without changing the scene the user requested. Yet stylistic rewriting can weaken object attributes and relations even when global text--image similarity remains high. We propose \ours{}, a structure-aware prompt editor that uses a calibrated scene graph as a shared representation for policy conditioning and image-grounded supervision. The graph exposes what should be preserved, while deletion, reordering, and insertion provide a token-aligned way to modify the prompt. This turns structural feedback from an external critique into a training signal for a reusable editing policy. Learning from that signal requires handling both uncertain grounding and stochastic image generation. \ours{} retains grounding confidence through smooth reward shaping and averages multiple images for each edit before computing group-relative advantages. Experiments span three disjoint 2K-prompt test sets and four frozen generation backbones, with comparisons to learned prompt optimizers, iterative structural feedback, diffusion-model fine-tuning, and inference-time attention control. \ours{} improves the faithfulness--aesthetics trade-off, with consistent gains in relational fidelity. Independent automatic evaluators and a 15-annotator study provide complementary evidence, and human preference margins increase with relation complexity. Controlled ablations connect these gains to graph conditioning, local editing, and image-grounded learning. The resulting policy produces reusable optimized text in one pass after graph calibration, without an iterative critique--rewrite loop.
\end{abstract}

\section{Introduction}

Text-to-image (T2I) models such as Stable Diffusion~\cite{stablediffusion} and DALL-E~\cite{dalle3} have made visual content creation widely accessible, but the quality of their outputs still depends strongly on prompt formulation. Users describe what they want to see; generators also respond to how that intent is expressed through composition, lighting, and stylistic detail. Automatic prompt optimization bridges this gap by transforming a user prompt $x$ into a generator-effective prompt $y$. Its objective is twofold: improve the resulting image's perceptual quality and preserve the scene specified by the user.

These goals need not move together. A rewriter can make a prompt more vivid while changing which objects appear, which attributes belong to them, or how they relate. Style-oriented optimizers~\cite{MagicPrompt,beautifulprompt} and embedding-level relevance constraints~\cite{promptist} address important parts of the task, but a high global text--image score can conceal local semantic failures~\cite{dsg,clipbad2}. A tiger with the wrong coat color or without its requested crown may remain topically similar to the prompt. Figure~\ref{img:intro} illustrates why \emph{preserving the scene} requires a more explicit account of intent than preserving overall similarity.

\begin{wrapfigure}{r}{0.48\textwidth}
\centering
\includegraphics[width=0.46\textwidth]{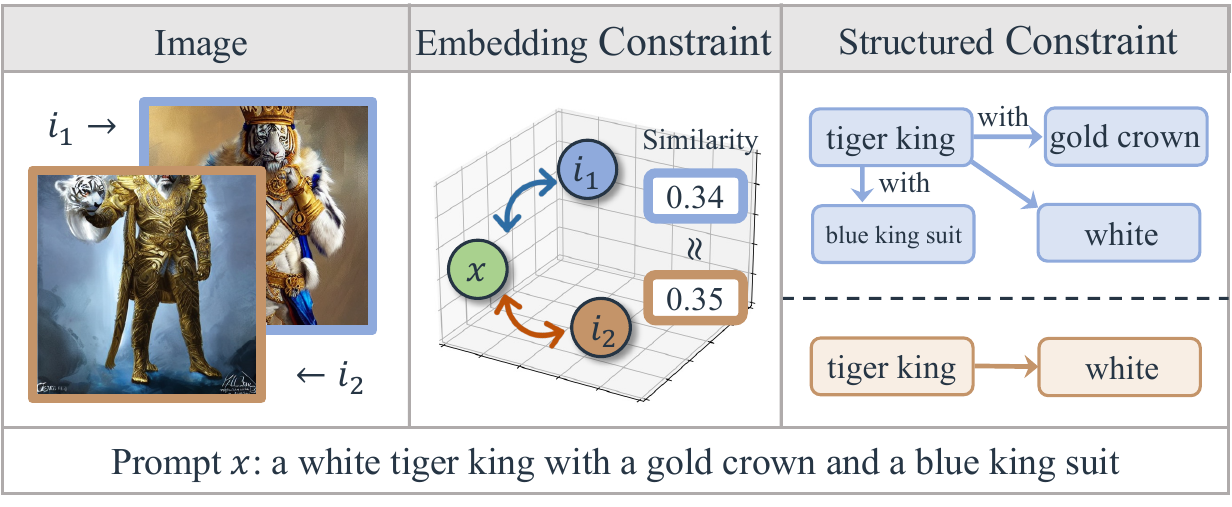}
\caption{Similar global embedding scores can conceal different structural errors. A scene graph makes object attributes and relations explicit.}
\label{img:intro}
\end{wrapfigure}

Scene graphs (SGs) express this intent through object phrases and relations. Complete phrases such as ``gold crown'' retain attributes, while edges specify how the objects are connected. Prior methods already exploit such structure: OPT2I optimizes consistency scores through iterative prompting~\cite{manasimproving}, and VisualPrompter uses atomic visual feedback to target missing concepts~\cite{visualprompter}. These methods establish the value of structural diagnosis. The next question is how to turn that diagnosis into a \emph{learned editing behavior} that generalizes across prompts, rather than repeatedly asking an external rewriter to interpret feedback at test time.

Our central idea is to use \emph{the same structured intent to condition the editor and to evaluate its consequences}. This connects the information available when choosing an edit to the object--attribute--relation constraints checked in the generated image. A graph-conditioned policy can then learn which textual changes improve rendering while retaining semantic anchors. Expressing those changes as deletion, reordering, and insertion further biases the policy toward modifying the input rather than replacing it wholesale. The benefit is not only fewer test-time revisions: it is a closer alignment between the representation of user intent, the granularity of prompt changes, and the feedback used to learn them.

Making this alignment learnable presents two challenges. First, text parsers may confuse style or camera descriptions with visible content, while grounding models provide uncertain evidence about genuine constraints. Second, image generation is stochastic: two renders of the same edited prompt can receive different scores. A policy trained on thresholded, single-image feedback can therefore react to evaluator discontinuities or rendering noise instead of the quality of an edit. Graph calibration, confidence-preserving rewards, and environment-sample averaging address these distinct sources of uncertainty.

We instantiate this approach as \textbf{\ours{}} (\textbf{Stru}ctured \textbf{P}rompt \textbf{O}ptimization). A scene-graph extractor calibrates the input structure, a hybrid encoder exposes text and graph features to the policy, and an edit-action generator produces the optimized prompt. Supervised fine-tuning provides an initial editing behavior; image-grounded reinforcement learning then improves its downstream consequences. Smooth structural rewards retain grounding confidence, and environment-aware GRPO compares edits after averaging multiple renders of each candidate. At inference time, one policy pass follows graph calibration and returns reusable text; the T2I generator stays frozen.

Our contributions are threefold. First, we develop a structure-aware editing policy in which calibrated intent is shared between policy conditioning and image-grounded supervision. Second, we pair confidence-preserving reward shaping with environment-aware group-relative optimization, supported by a reward-stability analysis. Third, we establish the approach's relation-sensitive benefits through controlled ablations, cross-backbone transfer, independent evaluators, and human judgments. The experiments connect improvements to graph access, local edits, and image feedback, while comparisons with learned and iterative optimizers characterize the resulting quality--cost trade-off. Additional evaluator and baseline results are provided in the appendix.

\section{Related Work}

\paragraph{Prompt optimization for T2I generation.}
Early work studies how users construct T2I prompts and derives practical design guidelines~\cite{prompt-ana,design-guildline}. Automatic methods optimize prompt tokens or embeddings through generator gradients in white-box settings~\cite{white-box1,white-box2}, or learn textual reformulations with supervised fine-tuning~\cite{beautifulprompt,neuroprompts}, multimodal supervision~\cite{pr1,pr2}, reinforcement learning~\cite{promptist}, and LLM prompting~\cite{chatgpt1}. Recent learned optimizers broaden the objective beyond style: PAE learns modifiers together with token weights and denoising-time ranges~\cite{pae}, while RePrompt uses GRPO to train reasoning-augmented reprompting from image-level feedback~\cite{reprompt}. \ours{} focuses on a complementary design choice: sharing a calibrated scene graph between the editor's input and its atomic grounding reward, then learning token-aligned edits from that feedback. Appendix~\ref{app:learned_baselines} compares PAE and RePrompt under matched SD v1.5/Lexica evaluation.

\paragraph{Structure-aware feedback for prompt optimization.}
OPT2I iteratively revises prompts to maximize consistency, including Davidsonian scene-graph (DSG) scores~\cite{manasimproving}. VisualPrompter decomposes prompts into atomic concepts, identifies missing content through visual feedback, and performs targeted refinement without training~\cite{visualprompter}. Both demonstrate that structural feedback can improve prompt--image alignment. \ours{} learns a reusable editor from this type of feedback: the graph conditions its decisions, and image-grounded rewards evaluate the edited prompt against the same constraints. The distinction is thus learned graph-conditioned editing versus test-time feedback-driven revision. Appendix~\ref{app:iterative_baseline} compares the resulting policy with a three-round GPT-4o structural-feedback system on the full 2K-prompt Lexica test set.

\paragraph{Scene graphs and fine-grained T2I evaluation.}
Scene graphs represent objects and pairwise relations~\cite{scene-graph-survey,visualgenome}. Image scene-graph generation combines object detection with contextual relation reasoning~\cite{impp,neuralmotifs,vctree,sgtr}; language-supervised systems extend grounding to open-vocabulary concepts~\cite{glip,language-supervised-sgg1}. For T2I evaluation, DSG~\cite{dsg} decomposes semantic correctness into dependent atomic questions, exposing failures that global similarity can miss. \ours{} uses object and relation grounding for policy learning while retaining attribute-bearing phrases in both graph features and grounding queries. We assess this learned objective with held-out OR/RR, and complement it with Gemini-VQA, human judgments, and independent evaluators (Appendix~\ref{app:external_metrics}).

\paragraph{Text editing and broader faithfulness interventions.}
Explicit text editors decompose rewriting into deletion, reordering, and insertion~\cite{felix,semi-edit}, providing a token-level link between an input and its revision. \ours{} adapts this action representation to graph-conditioned, image-grounded prompt optimization. Related approaches intervene elsewhere in the generation process: DPOK updates diffusion-model parameters~\cite{dpok}, and Attend-and-Excite adjusts the latent trajectory using attention-based guidance~\cite{attendandexcite}. These approaches expose different control and access requirements. Appendix~\ref{app:cross_paradigm} compares their quality, offline cost, and per-image interventions with the reusable-text interface of \ours{}.

\section{Proposed Method: \ours{}}

\paragraph{Problem formulation.}
Let $g$ denote a frozen T2I generator. Given a user prompt $x$, prompt optimization learns a policy $\pi_\theta$ that produces an optimized prompt $y$ such that images $i_y\sim g(y)$ improve perceptual quality while preserving the semantic intent of $x$. We formulate structure-aware prompt optimization as a single-step reinforcement-learning problem, or contextual bandit~\cite{contextual_bandit}. A calibrated scene graph $G_x=(O_x,V_x)$ represents object phrases, including their attributes, and relation triplets. The same $G_x$ conditions the policy and specifies the constraints used to score its output. The policy samples an edit-action sequence $a\sim\pi_\theta(\cdot\mid x,G_x)$ and deterministically applies it to obtain $y$. The environment then renders $i_y\sim g(y)$ and returns an image-grounded scalar reward $r_{x,y}$. The learning objective is $\mathcal{J}(\theta)=\mathbb{E}_{x,a,i_y}[r_{x,y}]$. Here, ``single-step'' refers to one complete prompt edit and its image reward, not to the token-level decoding steps within the editor.

\begin{figure*}[htb]
\centering
\includegraphics[width=0.96\textwidth]{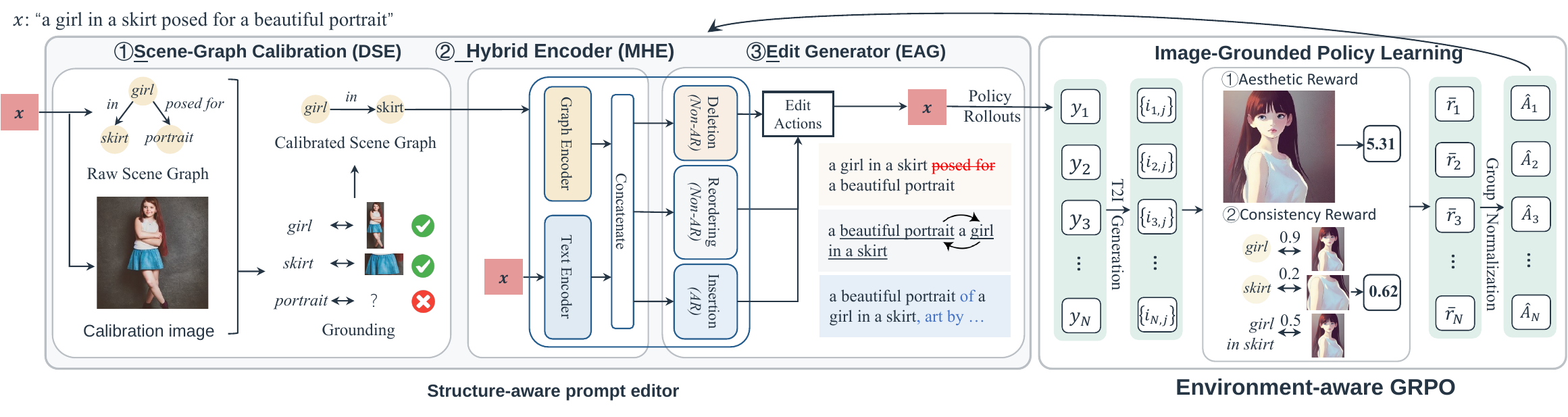}
\caption{Overview of \ours{}. Calibrated structural intent conditions the DSE--MHE--EAG editor and defines image-grounded consistency feedback. Environment-aware GRPO averages renders of each candidate before comparing edit rollouts.}
\label{img:framework}
\end{figure*}

\subsection{Scene-Graph-Aware Prompt Editing Policy}
Figure~\ref{img:framework} separates three roles: scene-graph extraction and calibration (DSE) establish the structural context; a multimodal hybrid encoder (MHE) makes it available to the policy; and an edit action generator (EAG) turns it into prompt changes. Together, they connect the user's semantic anchors to the decisions optimized by the image-grounded reward.

\paragraph{Scene-graph extraction and calibration (DSE).}
\label{para:method_archi_dsge}
DSE constructs a calibrated scene graph $G_x=(O_x,V_x)$ from the input prompt $x$. A language parser first produces a coarse graph $\tilde G_x=(\tilde O_x,\tilde V_x)$, but text-only parsing can introduce artifacts. For example, in ``\texttt{a girl photographed with a Leica lens},'' the phrase ``\texttt{Leica lens}'' may be misinterpreted as a visible object rather than a photographic condition.

DSE uses one image $i_x\sim g(x)$ to help identify ambiguous parser outputs. Parsed elements are partitioned into \emph{text-core} elements and \emph{auxiliary} elements. Text-core objects are explicit syntactic heads or entities participating in relations; auxiliary elements are ambiguous noun phrases that often describe style, camera, or medium. Visual evidence filters auxiliary elements, while text-core elements are retained even when the initial image fails to render them. Thus, a missing requested object remains a constraint for subsequent optimization; an ambiguous camera or style phrase can be removed from the graph.

Object grounding uses GLIP~\cite{glip}. Let $C_{\mathrm{GLIP}}(i_x)$ denote the detected open-vocabulary labels, and let $\tilde O_x^{\mathrm{core}}$ and $\tilde O_x^{\mathrm{aux}}$ denote the core and auxiliary object sets. We construct
\begin{equation*}
O_x=\tilde O_x^{\mathrm{core}}\cup\{o\in\tilde O_x^{\mathrm{aux}}\mid o\in C_{\mathrm{GLIP}}(i_x)\}.
\end{equation*}
For relation grounding, write each triplet as $v=(o_{\mathrm{sub}},p,o_{\mathrm{obj}})$, where $p$ is the predicate. Given endpoint boxes, we crop their minimal enclosing patch $\gamma^v_{i_x}$ and query BLIP~\cite{blip} with $q_v=o_{\mathrm{sub}}\oplus p\oplus o_{\mathrm{obj}}$. Core relations are retained, whereas auxiliary relations require positive visual support:
\begin{equation*}
V_x=\{v\in\tilde V_x\mid o_{\mathrm{sub}},o_{\mathrm{obj}}\in O_x,\; v\in\tilde V_x^{\mathrm{core}}\ \text{or}\ P_V(\mathrm{Yes}\mid q_v,\gamma^v_{i_x})>0.5\}.
\end{equation*}
Attributes are represented explicitly through complete phrases rather than separate typed nodes: modifiers remain attached in object labels or relation arguments, e.g., ``black hair,'' ``gold crown,'' and ``blue king suit.'' The same attribute-bearing phrases are therefore available to policy conditioning and image grounding. Appendix~\ref{app:attribute_ablation} isolates this pathway, while Appendix~\ref{app:dse_stress} evaluates DSE under injected parser noise.

\paragraph{Multimodal hybrid encoder (MHE).}
MHE maps the prompt and SG into a shared representation. A frozen CLIP text encoder~\cite{clip} produces token-level features $f_{\mathrm{text}}(x)\in\mathbb{R}^{l\times d}$. Each object phrase and serialized relation query $q_v$ is encoded by the same text encoder, with the \texttt{[EOS]} embedding used as the initial representation of a graph element. A two-layer GCN propagates information over the topology induced by $V_x$, producing $f_{\mathrm{graph}}(G_x)\in\mathbb{R}^{n\times d}$ for $n=|O_x|+|V_x|$. The editor receives the concatenated representation $f(x,G_x)=f_{\mathrm{text}}(x)\oplus f_{\mathrm{graph}}(G_x)$. This representation supplies token detail and graph context to the same editor. Appendix~\ref{app:clip_sensitivity} evaluates encoder sensitivity with a T5 replacement and a CLIP-free training variant.

\paragraph{Edit action generator (EAG).}
Represent $x$ as a token sequence $\{t_j\}_{j=1}^{l}$, with special start and end tokens. EAG predicts a compositional edit sequence from $f(x,G_x)$. Following explicit text-editing models~\cite{felix,semi-edit}, the actions are applied in the order illustrated in Figure~\ref{fig:edit_action}: deletion, reordering, and insertion.

\begin{figure}[htb]
\centering
\includegraphics[width=0.88\linewidth]{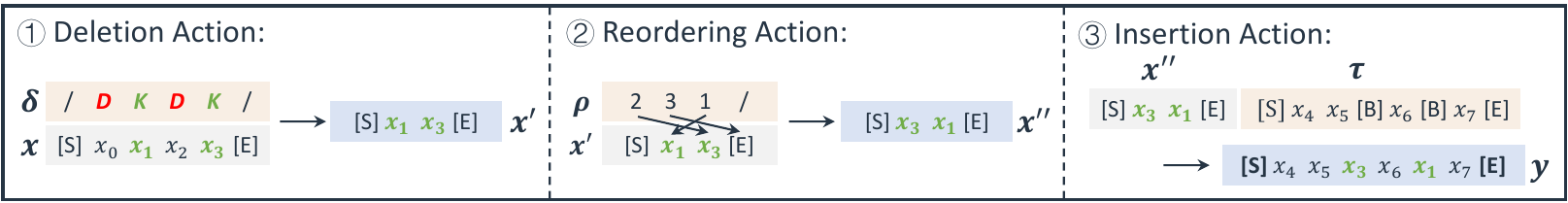}
\caption{The three edit operations used by EAG. Tokens are first retained or deleted, the retained subsequence is reordered, and new tokens are inserted into the resulting slots.}
\label{fig:edit_action}
\end{figure}

\begin{itemize}[leftmargin=1.35em,itemsep=1pt,topsep=2pt]
    \item \textit{Deletion.} For each token $t_j$, the deletion policy samples $\delta_j\in\{0,1\}$ from $\pi_{\theta_1}^{\mathrm{del}}(\delta_j\mid j,f(x,G_x))$, where $1$ retains the token and $0$ removes it. Boundary tokens are fixed.
    \item \textit{Reordering.} Let $x_{\mathrm{del}}=\{t_j\}_{j=1}^{l'}$ be the retained sequence. For each $j<l'$, the reordering policy predicts the successor distribution $\pi_{\theta_2}^{\mathrm{reo}}(\rho_j=k\mid j,x_{\mathrm{del}},f(x,G_x))$. A complete successor assignment induces a permutation decoded by beam search.
    \item \textit{Insertion.} Given the reordered sequence $x_{\mathrm{reo}}$, the insertion policy $\pi_{\theta_3}^{\mathrm{ins}}$ autoregressively samples new tokens or \texttt{[BOS]} slot markers. Retained tokens are filled into these slots in order, yielding the optimized prompt $y$.
\end{itemize}
The complete action is $a=(\delta,\rho,\tau)$. Deletion and reordering are decoded non-autoregressively, while insertion is autoregressive. This semi-autoregressive parameterization preserves an explicit correspondence to retained input tokens while allowing new phrasing and style. It is an inductive bias toward conservative edits, not a hard guarantee that every constraint is preserved; the image reward determines which edits are reinforced.

\subsection{Policy Training}
Training first teaches the policy how to express prompt revisions, then which revisions improve images. Supervised fine-tuning (SFT) warm-starts the editor from input--target pairs, and on-policy RL optimizes the rendered outcomes.

\paragraph{Supervised fine-tuning.}
For a paired dataset $\mathcal{D}=\{(x,y)\}$, token alignment yields a target edit sequence $a^\star(x,y)$: unmatched input tokens define deletions, retained tokens define successor pointers, and unmatched output tokens define insertions. We minimize
$\mathcal{L}_{\mathrm{SFT}}=-\mathbb{E}_{(x,y)\in\mathcal{D}}\log\pi_\theta(a^\star(x,y)\mid x,G_x)$.

\paragraph{Uncertainty-aware reward shaping.}
The reward combines structured consistency, aesthetic quality, and coarse image--text consistency. A common discrete structural score has the form $\mathbb{E}_{z\in G_x}\mathds{1}[P(z\mid i_y)\ge u]$, where $z$ is an object or relation and $u\in(0,1)$ is a fixed threshold. A small confidence change near $u$ can flip a constraint's contribution, whereas an improvement away from $u$ receives no credit. We instead retain grounding confidence. For $K=|O_x|+|V_x|\ge1$ constraints and confidence vector $p(I)\in[0,1]^K$, define the unshaped score $R_s(I)=K^{-1}\sum_k p_k(I)$. The following result characterizes its sensitivity to grounding perturbations.

\begin{theorem}
\label{theorem:smooth_reward}
For a rollout observation $I=(a,i_y)$, let $\widehat{\nabla}J(\theta;I)=\Psi(\theta;I)R_s(I)$, where $\Psi(\theta;I)=\nabla_\theta\log\pi_\theta(a\mid x,G_x)$. For any two rollouts $I,I'$ under the same context $(x,G_x)$,
\[
\|\widehat{\nabla}J(\theta;I)-\widehat{\nabla}J(\theta;I')\|
\le \|\Psi(\theta;I)-\Psi(\theta;I')\|+\frac{\|\Psi(\theta;I)\|}{K}\|p(I)-p(I')\|_1.
\]
A thresholded score generally has no uniform finite Lipschitz bound at $u$.
\end{theorem}
Theorem~\ref{theorem:smooth_reward} separates changes in the policy score function from changes in grounding confidence. For a fixed edit, the first term vanishes: confidence perturbations affect the gradient contribution continuously. Appendix~\ref{app:theory} gives the proof and a reward-variation bound.

In training, we use the monotone clipped-log transform $\phi(t)=[\log(\max\{t,\eta\})]_{c_{\min}}^{c_{\max}}$, with $\eta>0$ and $[s]_a^b=\min\{b,\max\{a,s\}\}$. This bounded transform retains a finite Lipschitz constant (Appendix~\ref{app:theory}); its constants are fixed on validation data. The structured reward is
\begin{equation*}
r_{x,y}^{\mathrm{con}}=
\frac{\sum_{o\in O_x}\phi(P_O(o\mid i_y))+\sum_{v\in V_x}\phi(P_V(\mathrm{Yes}\mid q_v,\gamma^v_{i_y}))}{\max(1,K)}.
\end{equation*}
GLIP supplies object-phrase confidence and BLIP verifies relation-specific crops, so atomic consistency is scored independently of CLIP. Two frozen auxiliary rewards complete the objective. Following Promptist~\cite{promptist}, $r_{x,y}^{\mathrm{aes}}$ uses the released improved aesthetic predictor, and $r_{x,y}^{\mathrm{CLIP}}$ regularizes global prompt--image consistency. The final reward is $r_{x,y}=\alpha_1r_{x,y}^{\mathrm{con}}+\alpha_2r_{x,y}^{\mathrm{aes}}+\alpha_3r_{x,y}^{\mathrm{CLIP}}$, with weights selected on a disjoint validation split. Appendix~\ref{app:implementation} gives the predictor specification and reward weights.

\paragraph{Environment-aware GRPO.}
We optimize the policy with GRPO~\cite{grpo}. Rendering remains stochastic even when grounding scores are smoothed. To compare edits rather than individual lucky or unlucky renders, we average environment samples before group normalization. For each prompt $x$, we sample $N$ edit rollouts $a_i\sim\pi_{\theta_{\mathrm{old}}}(\cdot\mid x,G_x)$. For every $a_i$, we render $M$ images, obtain rewards $r_{i,j}$, and compute $\bar r_i=M^{-1}\sum_{j=1}^{M}r_{i,j}$. The normalized advantage is
$\widehat A_i=(\bar r_i-\operatorname{mean}_k\bar r_k)/(\operatorname{std}_k\bar r_k+\epsilon_A)$.
The clipped objective is
\begin{align}
\mathcal{J}_{\mathrm{GRPO}}(\theta)=\mathbb{E}\frac1N\sum_i\Big[
\min\big(w_i\widehat A_i,\operatorname{clip}(w_i,1-\epsilon,1+\epsilon)\widehat A_i\big)
-\beta D_{\mathrm{KL}}(\pi_\theta\|\pi_{\mathrm{ref}})+\lambda\mathcal{H}(\pi_\theta)\Big],
\end{align}
where $w_i=\pi_\theta(a_i\mid x,G_x)/\pi_{\theta_{\mathrm{old}}}(a_i\mid x,G_x)$. The KL term anchors the policy to its SFT reference, while entropy regularization maintains exploration. Smoothing controls sensitivity within a render; multi-image averaging reduces sampling variation across renders of the same edit. Algorithm~\ref{algorithm:rl} in Appendix~\ref{appendix:rl} gives the complete training procedure.

\section{Experiments}

\subsection{Experimental Settings}

\paragraph{Data collection.}
For SFT, we collect 72K optimized prompts from Lexica as targets. Following Promptist~\cite{promptist}, each optimized prompt is converted into two input variants: (i) a semantic skeleton obtained by removing stylistic modifiers and (ii) a GPT-4o rewrite of that skeleton. This procedure yields 144K input--target pairs. SFT uses the full paired set, whereas RL uses 24K input prompts sampled from the training split without access to their targets. Evaluation uses three disjoint test sets: Lexica and DiffusionDB~\cite{diffusiondb} as in-domain prompt-optimization distributions, and COCO captions~\cite{coco} as an out-of-domain distribution. Each test set contains 2K prompts and is disjoint from training, validation, and hyperparameter selection.

\paragraph{Implementation details.}
Unless stated otherwise, we use GLIP-T~\cite{glip} for object grounding, BLIP-base~\cite{blip} for relation verification, Stable Diffusion v1.5~\cite{stablediffusion} as the default generator, and CLIP ViT-L/14~\cite{clip} as the frozen text encoder. The edit policy is initialized from BART-base~\cite{lewis-etal-2020-bart}; the graph encoder is a two-layer GCN with hidden size 768. Grounding models, CLIP, and the T2I generator remain frozen, and only the GCN and editing policy are updated. We use AdamW~\cite{adamw} with learning rate $2\times10^{-5}$. SFT runs for approximately 50 epochs, followed by 24K RL prompts organized into 6K updates. Each RL context samples $N=4$ policy rollouts and $M=3$ images per rollout. Hyperparameters are selected once on a held-out Lexica validation split and then fixed across all test distributions and generation backbones. Appendix~\ref{app:implementation} provides model versions, the aesthetic-predictor checkpoint and architecture, optimization settings, and compute accounting.

\paragraph{Evaluation protocol.}
We evaluate perceptual quality and semantic faithfulness. AES$_\Delta$ denotes the change in the frozen aesthetic score relative to the raw prompt. For structured consistency, we evaluate held-out images with thresholded grounding rather than the smoothed training reward, reporting object recall (OR), relation recall (RR), and $\mathrm{SGR}=(\mathrm{OR}+\mathrm{RR})/2$. Unless noted, each score averages three matched-seed generations per prompt. We compare \ours{} with template prompting~\cite{bestprompt}, learned LM optimizers (Promptist~\cite{promptist}, MagicPrompt~\cite{MagicPrompt}, and BeautifulPrompt~\cite{beautifulprompt}), structure-aware feedback prompting (VisualPrompter~\cite{visualprompter}), GPT rephrasing and few-shot prompting, and human-crafted target prompts when available. The Template baseline appends a universal style-only suffix while leaving semantic content unchanged, so it serves as a near-original-prompt reference. OR/RR measure generalization on the trained structural objective; Gemini-VQA and human judgments assess faithfulness through separate evaluators. TIFA, VQAScore, ImageReward, HPSv2.1, and further evaluation details appear in Appendix~\ref{app:external_metrics}.

\subsection{Main Results}
We first evaluate the quality--faithfulness trade-off, then test transfer across generators and perceptibility to human annotators. The ablations that follow isolate the representation and learning choices behind these outcomes.

\begin{table*}[htb]
\centering
\caption{Overall performance with Stable Diffusion v1.5. Gemini-VQA is evaluated on Lexica only and is never used for training or model selection. Best values are bold. VisualPrompter attains the highest COCO SGR but decreases aesthetics; \ours{} obtains the strongest in-domain faithfulness--aesthetics trade-off and the largest AES$_\Delta$ on all three datasets.}
\label{tab:exp_main}
\resizebox{\textwidth}{!}{
\scriptsize
\begin{tabular}{l|ccccc|cccc|cccc}
\toprule
\multirow{2}{*}{\textbf{Method}}&\multicolumn{5}{c|}{\textbf{Lexica}: In-domain}&\multicolumn{4}{c|}{\textbf{DiffusionDB}: In-domain}&\multicolumn{4}{c}{\textbf{COCO}: Out-of-domain} \\
& AES$_\Delta$ & OR(\%) & RR(\%) & SGR(\%) & Gemini & AES$_\Delta$ & OR(\%) & RR(\%) & SGR(\%) & AES$_\Delta$ & OR(\%) & RR(\%) & SGR(\%) \\
\midrule
Template & 0.53 & 89.50 & 59.63 & 74.57 & 77.6 & 0.75 & 91.03 & 62.65 & 76.84 & 0.68 & 90.18 & 65.02 & 77.60\\
Promptist & 0.48 & 89.55 & 61.25 & 75.40 & 79.5 & 0.64 & 92.17 & 60.14 & 76.16 & 0.80 & 90.06 & 64.61 & 77.33\\
MagicPrompt & 0.46 & 89.50 & 61.67 & 75.59 & 78.1 & 0.62 & 92.75 & 58.92 & 75.84 & 0.67 & 89.17 & 59.12 & 74.14\\
Human-crafted & 0.54 & 89.19 & 56.31 & 72.75 & 78.2 & 0.42 & 92.32 & 59.44 & 75.88 & -- & -- & -- & --\\
VisualPrompter & 0.62 & 90.26 & 53.49 & 71.88 & 79.2 & 0.33 & 83.11 & 23.43 & 53.27 & -0.36 & 90.50 & \textbf{76.01} & \textbf{83.26}\\
BeautifulPrompt & 0.57 & 85.70 & 39.85 & 62.78 & 73.8 & 0.79 & 89.37 & 37.61 & 63.49 & 0.69 & 81.91 & 35.27 & 58.59\\
Rephrase (GPT-3.5) & 0.43 & 89.82 & 61.68 & 75.74 & 78.4 & 0.59 & 91.73 & 62.42 & 77.07 & 0.54 & 89.54 & 63.54 & 76.54\\
Rephrase (GPT-4o) & 0.44 & 88.38 & 55.52 & 71.95 & 78.6 & 0.63 & 91.82 & 56.13 & 73.97 & 0.67 & 90.53 & 66.95 & 78.74\\
Few-shot (GPT-4o) & 0.51 & 89.27 & 60.58 & 74.92 & 78.0 & 0.66 & 90.95 & 62.75 & 76.85 & 0.67 & 90.38 & 66.57 & 78.48\\
\midrule
\rowcolor{mygray}\textbf{\ours{} (Ours)} & \textbf{0.63} & \textbf{91.84} & \textbf{69.64} & \textbf{80.74} & \textbf{82.4} & \textbf{0.91} & \textbf{92.94} & \textbf{65.49} & \textbf{79.22} & \textbf{1.02} & \textbf{91.28} & 66.59 & 78.94\\
\bottomrule
\end{tabular}}
\end{table*}

Table~\ref{tab:exp_main} shows that the clearest gains arise in relation recall. On Lexica, \ours{} improves RR by 7.96 points over the best listed baseline and raises SGR to 80.74 while also improving aesthetics. On DiffusionDB, it achieves the best AES$_\Delta$, OR, RR, and SGR. The larger separation in RR than in OR matches the method's emphasis: retaining object words alone does not preserve the relations between them. The held-out Gemini-VQA column follows the same Lexica trend and ranks \ours{} highest, complementing the more localized grounding metrics.

On out-of-domain COCO, the methods occupy different operating points. VisualPrompter obtains the highest SGR, driven by RR, but its aesthetic score decreases by 0.36. In contrast, \ours{} yields the largest aesthetic improvement and retains structured consistency comparable to the GPT-4o baselines. The pattern separates aesthetic improvement from literal constraint satisfaction: VisualPrompter favors relation recall, whereas \ours{} combines a positive aesthetic shift with competitive consistency. The independent metrics in Appendix~\ref{app:external_metrics} show the same distinction on COCO.

\paragraph{Generalization across T2I backbones.}
\begin{table*}[htb]
\centering
\caption{Transferability across T2I backbones on Lexica. The editor is trained with SD v1.5 feedback and evaluated without retraining on SD v1.4, SD v2.1, and FLUX.1-dev.}
\label{tab:exp_generalize}
\resizebox{0.96\textwidth}{!}{
\scriptsize
\begin{tabular}{l|cccc|cccc|cccc}
\toprule
\multirow{2}{*}{\textbf{Method}}&\multicolumn{4}{c|}{\textbf{Stable Diffusion v1.4}}&\multicolumn{4}{c|}{\textbf{Stable Diffusion v2.1}}&\multicolumn{4}{c}{\textbf{FLUX.1-dev}}\\
& AES$_\Delta$ & OR(\%) & RR(\%) & SGR(\%) & AES$_\Delta$ & OR(\%) & RR(\%) & SGR(\%) & AES$_\Delta$ & OR(\%) & RR(\%) & SGR(\%)\\
\midrule
Template & .53 & 89.19 & 58.38 & 73.78 & .33 & 88.02 & 59.25 & 73.64 & .83 & 87.88 & 46.03 & 66.96\\
Promptist & .47 & \textbf{89.75} & 60.64 & 75.19 & .34 & 87.75 & 59.57 & 73.66 & .88 & 90.17 & 53.41 & 71.79\\
MagicPrompt & .48 & 89.44 & 60.18 & 74.81 & .32 & 87.98 & 55.65 & 71.81 & .82 & 90.16 & 52.81 & 71.48\\
Human-crafted & .56 & 89.56 & 59.48 & 74.52 & .39 & 87.75 & 54.28 & 71.01 & .62 & 87.92 & 49.55 & 68.74\\
VisualPrompter & .62 & 88.10 & 47.05 & 67.57 & \textbf{.62} & 88.07 & 46.70 & 67.38 & .63 & 88.08 & 46.53 & 67.31\\
BeautifulPrompt & .59 & 85.90 & 41.05 & 63.47 & .42 & 84.24 & 37.54 & 60.89 & .94 & 89.83 & 30.73 & 60.28\\
Few-shot (GPT-4o) & .44 & 89.05 & 59.13 & 74.09 & .29 & 87.65 & 55.73 & 71.69 & .83 & 87.82 & 46.99 & 67.41\\
\midrule
\rowcolor{mygray}\textbf{\ours{} (Ours)} & \textbf{.68} & 89.18 & \textbf{61.82} & \textbf{75.50} & .46 & \textbf{89.14} & \textbf{61.57} & \textbf{75.36} & \textbf{1.14} & \textbf{91.41} & \textbf{62.43} & \textbf{76.92}\\
\bottomrule
\end{tabular}}
\end{table*}

Table~\ref{tab:exp_generalize} tests whether the learned editor overfits to idiosyncrasies of SD v1.5. Without retraining, \ours{} achieves the highest RR and SGR on all three unseen generators. On FLUX.1-dev, for example, RR is 62.43 versus Promptist's 53.41, alongside a larger aesthetic gain. VisualPrompter has higher AES$_\Delta$ on SD v2.1, while \ours{} has higher structural recall. This transfer pattern supports reuse of the learned textual edits across generators.

\paragraph{Human evaluation.}
\label{para:human_eval}
We conduct a pairwise study with 15 annotators over all 2K Lexica prompts. Annotators see the original prompt and two anonymized matched-seed images in randomized left--right order, then judge aesthetics and semantic consistency separately. Annotators can select a tie when neither image is preferable; ties are retained without a secondary tie-breaking rule. We aggregate each pair by majority vote. Agreement is moderate, with Fleiss' $\kappa=.40$ for aesthetics and $.43$ for consistency. Table~\ref{tab:humaneval} reports the resulting win/tie/loss percentages for \ours{}.

\begin{table}[htb]
\centering
\caption{Human preference on Lexica. Bars show \ours{} wins (blue), ties (gray), and losses (sand), from left to right, with the percentages printed in each segment.}
\label{tab:humaneval}
\resizebox{0.80\linewidth}{!}{
\begin{tabular}{lcc}
\toprule
\textbf{Comparison} & \textbf{Aesthetics} \textit{(win/tie/loss)} & \textbf{Consistency} \textit{(win/tie/loss)}\\
\midrule
\ours{} vs. Promptist
& \raisebox{-0.15\height}{\includegraphics[width=0.3\textwidth]{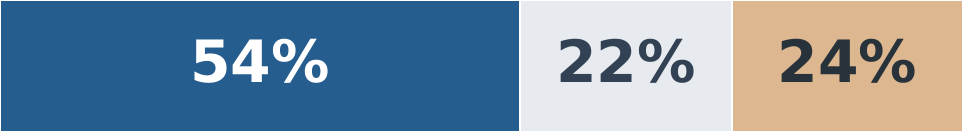}}
& \raisebox{-0.15\height}{\includegraphics[width=0.3\textwidth]{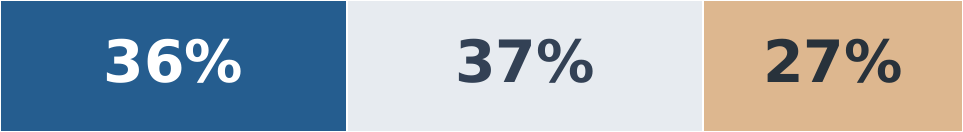}}\\
\ours{} vs. MagicPrompt
& \raisebox{-0.15\height}{\includegraphics[width=0.3\textwidth]{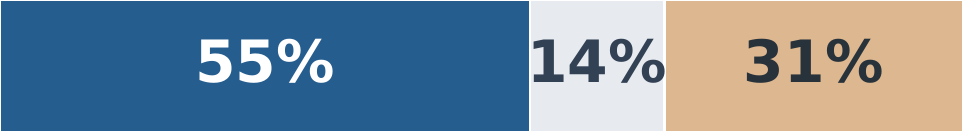}}
& \raisebox{-0.15\height}{\includegraphics[width=0.3\textwidth]{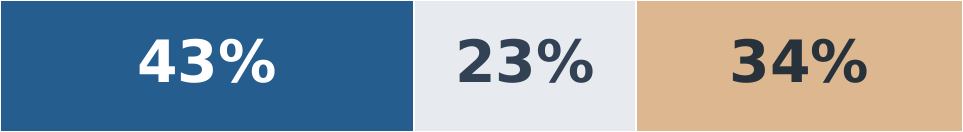}}\\
\ours{} vs. VisualPrompter
& \raisebox{-0.15\height}{\includegraphics[width=0.3\textwidth]{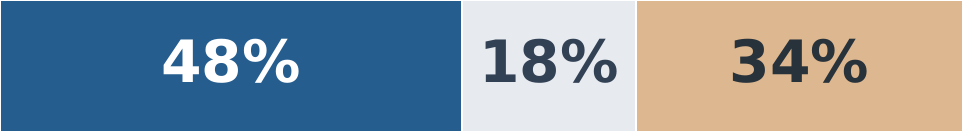}}
& \raisebox{-0.15\height}{\includegraphics[width=0.3\textwidth]{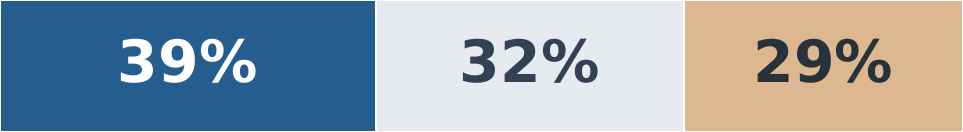}}\\
\ours{} vs. BeautifulPrompt
& \raisebox{-0.15\height}{\includegraphics[width=0.3\textwidth]{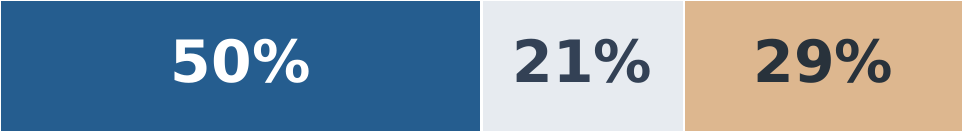}}
& \raisebox{-0.15\height}{\includegraphics[width=0.3\textwidth]{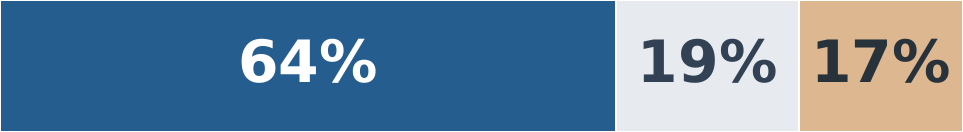}}\\
\ours{} vs. Few-shot GPT-4o
& \raisebox{-0.15\height}{\includegraphics[width=0.3\textwidth]{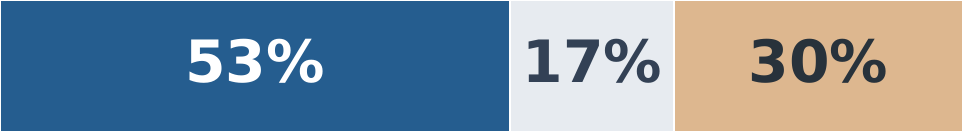}}
& \raisebox{-0.15\height}{\includegraphics[width=0.3\textwidth]{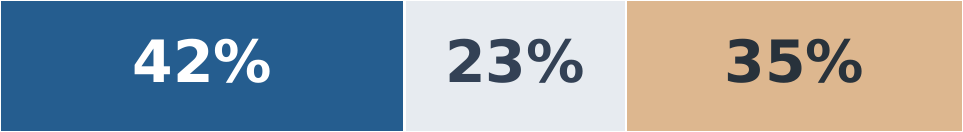}}\\
\bottomrule
\end{tabular}}
\end{table}

In Table~\ref{tab:humaneval}, \ours{} has a positive win--loss margin for both criteria against all five comparators. Consistency margins against Promptist, VisualPrompter, and few-shot GPT-4o are $+9$, $+10$, and $+7$ percentage points, respectively, with ties in 23--37\% of comparisons. All three margins have positive 95\% confidence intervals. Crucially, the advantage becomes more perceptible as relations accumulate: from $|V_x|=1$ to $|V_x|\ge3$, the corresponding margins rise from $+6$ to $+16$, $+6$ to $+17$, and $+3$ to $+14$ points. This links the human result to the compositional difficulty targeted by the method. Appendix~\ref{app:human_details} reports confidence intervals, sign tests, forced-choice replication, and the DiffusionDB/COCO studies.

\subsection{Ablation Study}

\begin{table}[htb]
\centering
\caption{Ablation study on Lexica. DSE, MHE, and EAG test the architecture; ``w/o RL'' and ``w/o $r^{\mathrm{con}}$'' test image-grounded optimization. The w/o-EAG variant replaces local editing with free-form autoregressive generation while retaining the calibrated SG, reward, and GRPO.}
\label{tab:ablation}
\resizebox{0.64\linewidth}{!}{
\begin{tabular}{l|ccc|cccc}
\toprule
\textbf{Variant} & DSE & MHE & EAG & AES$_\Delta$ & OR & RR & SGR\\
\midrule
\rowcolor{mygray}\ours{} & \cmark & \cmark & \cmark & .63 & \textbf{91.84} & \textbf{69.64} & \textbf{80.74}\\
w/o DSE & \xmark & \cmark & \cmark & .64 & 89.24 & 58.14 & 73.69\\
w/o MHE & \cmark & \xmark & \cmark & .65 & 89.25 & 61.41 & 75.33\\
w/o EAG & \cmark & \cmark & \xmark & .57 & 88.92 & 58.44 & 73.68\\
w/o RL & \cmark & \cmark & \cmark & .57 & 89.52 & 58.55 & 74.04\\
w/o $r^{\mathrm{con}}$ & \cmark & \cmark & \cmark & \textbf{.68} & 89.19 & 61.08 & 75.14\\
\bottomrule
\end{tabular}}
\end{table}

Table~\ref{tab:ablation} tests whether the gains depend on structural input, the editing parameterization, and image feedback. Removing DSE introduces uncalibrated parser noise and reduces SGR by 7.05 points. Removing MHE prevents the editor from directly observing graph structure and lowers SGR by 5.41 points. Replacing EAG with unrestricted generation causes a 7.06-point drop despite retaining graph conditioning, the reward, and GRPO. Thus, the action representation contributes beyond simply supplying a structural reward to a rewriter. Removing RL lowers SGR by 6.70 points, showing the benefit of optimizing rendered outcomes beyond supervised imitation. Removing $r^{\mathrm{con}}$ increases aesthetics slightly but lowers SGR by 5.60 points, separating visual appeal from structural fidelity. Attribute-specific, CLIP-sensitivity, reward-weight, and parser-perturbation ablations appear in Appendices~\ref{app:attribute_ablation}, \ref{app:clip_sensitivity}, \ref{app:reward_sensitivity}, and~\ref{app:dse_stress}, respectively.

\paragraph{Effect of constraint complexity.}
\begin{figure*}[htb]
\centering
\includegraphics[width=0.96\linewidth]{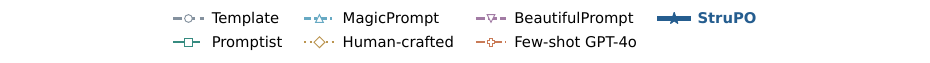}
\par
\begin{subfigure}[t]{0.485\textwidth}
    \centering
    \includegraphics[width=\linewidth]{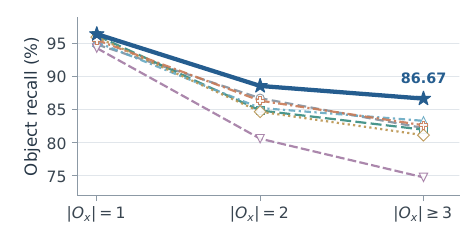}
    \caption{Object complexity.}
    \label{fig:object_complexity}
\end{subfigure}
\hfill
\begin{subfigure}[t]{0.485\textwidth}
    \centering
    \includegraphics[width=\linewidth]{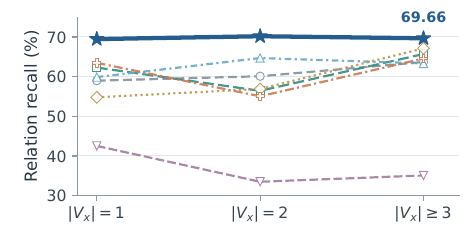}
    \caption{Relation complexity.}
    \label{fig:relation_complexity}
\end{subfigure}
\caption{Performance by constraint complexity on Lexica. Object and relation counts define separate partitions; methods share prompts and image seeds within each bucket.}
\label{fig:constraint_complexity}
\end{figure*}

Figure~\ref{fig:constraint_complexity} evaluates robustness to increasingly dense semantic constraints. In panel~\subref{fig:object_complexity}, OR declines for every method as the object count grows, but \ours{} shows the smallest degradation. Panel~\subref{fig:relation_complexity} is more revealing: generic rewriting methods fluctuate substantially as relation count changes, whereas \ours{} remains close to 70\% RR across all three buckets. The larger human preference margins in the high-relation bucket (Table~\ref{tab:human_relation_complexity}) provide a perceptual counterpart to this structural robustness.

\paragraph{Additional analyses.}
Matched learned and iterative baselines test the policy against alternative ways of using image feedback. Attribute/CLIP and reward/DSE diagnostics examine the structural pathway, while runtime and failure analyses characterize deployment trade-offs (Appendices~\ref{app:external_metrics}--\ref{app:failure_cases}).

\subsection{Qualitative Analysis}
\begin{figure*}[htb]
\centering
\includegraphics[width=0.90\textwidth]{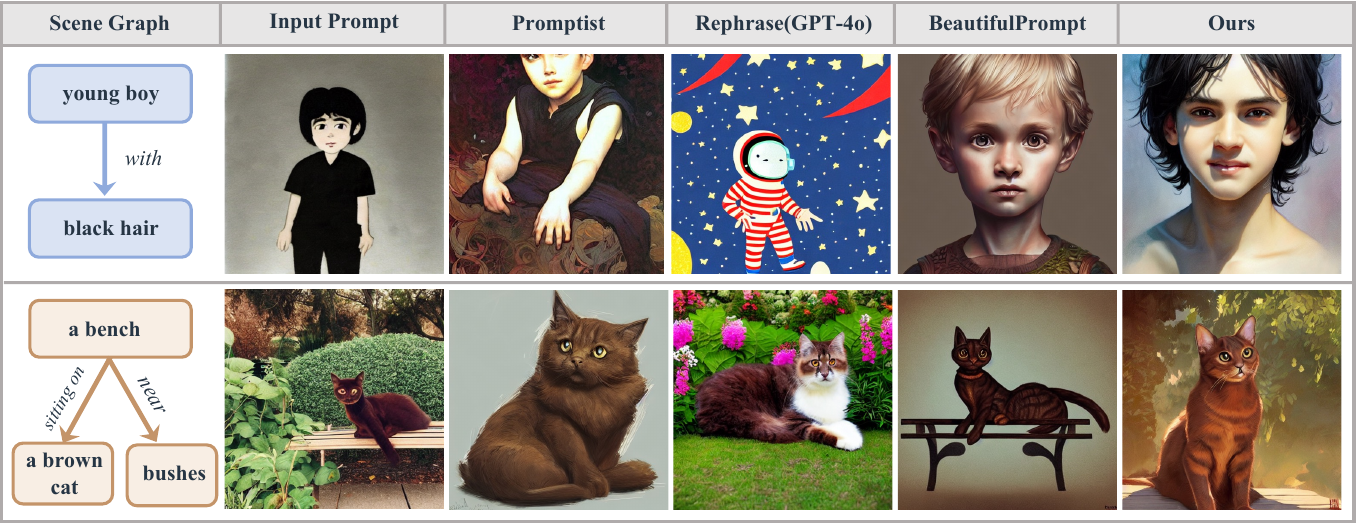}
\caption{Representative SD v1.5 comparisons. \ours{} better preserves attribute-bearing phrases and object--relation constraints while maintaining visual quality.}
\label{fig:image_case}
\end{figure*}

Figure~\ref{fig:image_case} illustrates the semantic differences reflected in the aggregate scores. In the first case, several baselines omit or alter ``black hair,'' either changing the attribute or the subject itself. In the second, baselines miss an object or fail to place the objects in the requested relation. The resulting images may remain individually plausible and topically relevant, yet they are not equally faithful to the user's structured intent. The graphs make these differences explicit: modifier attachment and relation direction remain identifiable even when all images depict a plausible scene. In these cases, \ours{} retains the relevant phrases and relations while adding visual detail. The examples complement the ablations by showing the kinds of local semantic errors that aggregate object recall alone can obscure.

\FloatBarrier

\section{Conclusions and Limitations}
\ours{} learns faithful prompt editing by sharing calibrated structural intent between policy conditioning and image-grounded supervision. Local edit actions preserve a token-level connection to the user's description, while confidence-preserving rewards and environment-aware GRPO make noisy structural feedback usable for optimization. Across prompt distributions and frozen T2I backbones, the resulting policy improves the faithfulness--aesthetics trade-off, with the most consistent gains in relations. Controlled ablations connect these gains to the graph, action representation, and image feedback; independent evaluators and human judgments provide complementary evidence. The broader lesson is that structural feedback can shape how a prompt policy learns, not merely how its outputs are scored. After graph calibration, the editor returns reusable optimized text in one pass, linking relation-sensitive quality to a practical frozen-generator interface.

\paragraph{Limitations.}
The approach depends on parser, grounding, and generator quality. Attribute-bearing phrases capture local modifiers, but object--relation graphs are less expressive for counting, occlusion, negation, emotion, metaphor, and other abstract intent. Conservative editing can also be too restrictive when a prompt needs global reformulation; the long-input and dense-scene analyses expose these limits. Multi-image averaging increases offline training cost, and inference includes an initial calibration render in addition to the final image. Transfer does not remove the capabilities or biases of the underlying generator, and the studied English prompt distributions do not establish multilingual coverage. Richer constraint representations and adaptive editing strategies are promising next steps, alongside broader human evaluation of which semantic differences matter in practice.

\begin{ack}
This work was also supported by the Noncommunicable Chronic Diseases-National
Science and Technology Major Project (Grant No.~2026ZD0558202).
Haobo Wang is also supported by the Ningbo Yongjiang Talent Introduction
Programme (No.~2023A-399-G).
\end{ack}


\bibliographystyle{unsrtnat}
\bibliography{reference}

@inproceedings{stablediffusion,
  title={High-resolution image synthesis with latent diffusion models},
  author={Rombach, Robin and Blattmann, Andreas and Lorenz, Dominik and Esser, Patrick and Ommer, Bj{\"o}rn},
  booktitle={Proceedings of the IEEE/CVF conference on computer vision and pattern recognition},
  pages={10684--10695},
  year={2022}
}

@article{dalle3,
  title={Improving image generation with better captions},
  author={Betker, James and Goh, Gabriel and Jing, Li and Brooks, Tim and Wang, Jianfeng and Li, Linjie and Ouyang, Long and Zhuang, Juntang and Lee, Joyce and Guo, Yufei and others},
  journal={Computer Science. https://cdn. openai. com/papers/dall-e-3. pdf},
  volume={2},
  number={3},
  year={2023}
}

@inproceedings{promptist,
  title={Optimizing Prompts for Text-to-Image Generation},
  author={Hao, Yaru and Chi, Zewen and Dong, Li and Wei, Furu},
  booktitle={Advances in Neural Information Processing Systems},
  volume={36},
  pages={66923--66939},
  year={2023}
}

@article{neuroprompts,
  title={NeuroPrompts: An Adaptive Framework to Optimize Prompts for Text-to-Image Generation},
  author={Rosenman, Shachar and Lal, Vasudev and Howard, Phillip},
  journal={arXiv preprint arXiv:2311.12229},
  year={2023}
}

@article{beautifulprompt,
  title={BeautifulPrompt: Towards Automatic Prompt Engineering for Text-to-Image Synthesis},
  author={Cao, Tingfeng and Wang, Chengyu and Liu, Bingyan and Wu, Ziheng and Zhu, Jinhui and Huang, Jun},
  journal={arXiv preprint arXiv:2311.06752},
  year={2023}
}

@article{coco,
  title={Microsoft coco captions: Data collection and evaluation server},
  author={Chen, Xinlei and Fang, Hao and Lin, Tsung-Yi and Vedantam, Ramakrishna and Gupta, Saurabh and Doll{\'a}r, Piotr and Zitnick, C Lawrence},
  journal={arXiv preprint arXiv:1504.00325},
  year={2015}
}

@inproceedings{prompt-ana,
  title={A Prompt Log Analysis of Text-to-Image Generation Systems},
  author={Xie, Yutong and Pan, Zhaoying and Ma, Jinge and Jie, Luo and Mei, Qiaozhu},
  booktitle={Proceedings of the ACM Web Conference 2023},
  pages={3892--3902},
  year={2023}
}

@inproceedings{design-guildline,
  title={Design guidelines for prompt engineering text-to-image generative models},
  author={Liu, Vivian and Chilton, Lydia B},
  booktitle={Proceedings of the 2022 CHI Conference on Human Factors in Computing Systems},
  pages={1--23},
  year={2022}
}

@article{white-box1,
  title={Hard prompts made easy: Gradient-based discrete optimization for prompt tuning and discovery},
  author={Wen, Yuxin and Jain, Neel and Kirchenbauer, John and Goldblum, Micah and Geiping, Jonas and Goldstein, Tom},
  journal={Advances in Neural Information Processing Systems},
  volume={36},
  year={2024}
}

@article{white-box2,
  title={Manipulating Embeddings of Stable Diffusion Prompts},
  author={Deckers, Niklas and Peters, Julia and Potthast, Martin},
  journal={arXiv preprint arXiv:2308.12059},
  year={2023}
}

@article{chatgpt1,
  title={Could ChatGPT Imagine: Content Control for Artistic Painting Generation Via Large Language Models},
  author={Lu, Yue and Guo, Chao and Dou, Yong and Dai, Xingyuan and Wang, Fei-Yue},
  journal={Journal of Intelligent \& Robotic Systems},
  volume={109},
  number={2},
  pages={39},
  year={2023},
  publisher={Springer}
}

@inproceedings{language-supervised-sgg1,
  title={Learning To Generate Language-Supervised and Open-Vocabulary Scene Graph Using Pre-Trained Visual-Semantic Space},
  author={Zhang, Yong and Pan, Yingwei and Yao, Ting and Huang, Rui and Mei, Tao and Chen, Chang-Wen},
  booktitle={Proceedings of the IEEE/CVF Conference on Computer Vision and Pattern Recognition},
  pages={2915--2924},
  year={2023}
}

@inproceedings{glip,
  title={Grounded language-image pre-training},
  author={Li, Liunian Harold and Zhang, Pengchuan and Zhang, Haotian and Yang, Jianwei and Li, Chunyuan and Zhong, Yiwu and Wang, Lijuan and Yuan, Lu and Zhang, Lei and Hwang, Jenq-Neng and others},
  booktitle={Proceedings of the IEEE/CVF Conference on Computer Vision and Pattern Recognition},
  pages={10965--10975},
  year={2022}
}

@inproceedings{blip,
  title={Blip: Bootstrapping language-image pre-training for unified vision-language understanding and generation},
  author={Li, Junnan and Li, Dongxu and Xiong, Caiming and Hoi, Steven},
  booktitle={International Conference on Machine Learning},
  pages={12888--12900},
  year={2022},
  organization={PMLR}
}

@article{semi-edit,
  title={Edit5: Semi-autoregressive text-editing with t5 warm-start},
  author={Mallinson, Jonathan and Adamek, Jakub and Malmi, Eric and Severyn, Aliaksei},
  journal={arXiv preprint arXiv:2205.12209},
  year={2022}
}

@article{felix,
  title={FELIX: Flexible text editing through tagging and insertion},
  author={Mallinson, Jonathan and Severyn, Aliaksei and Malmi, Eric and Garrido, Guillermo},
  journal={arXiv preprint arXiv:2003.10687},
  year={2020}
}

@article{diffusiondb,
  title={Diffusiondb: A large-scale prompt gallery dataset for text-to-image generative models},
  author={Wang, Zijie J and Montoya, Evan and Munechika, David and Yang, Haoyang and Hoover, Benjamin and Chau, Duen Horng},
  journal={arXiv preprint arXiv:2210.14896},
  year={2022}
}

@inproceedings{clip,
  title={Learning transferable visual models from natural language supervision},
  author={Radford, Alec and Kim, Jong Wook and Hallacy, Chris and Ramesh, Aditya and Goh, Gabriel and Agarwal, Sandhini and Sastry, Girish and Askell, Amanda and Mishkin, Pamela and Clark, Jack and others},
  booktitle={International conference on machine learning},
  pages={8748--8763},
  year={2021},
  organization={PMLR}
}

@inproceedings{bestprompt,
  title={Best prompts for text-to-image models and how to find them},
  author={Pavlichenko, Nikita and Ustalov, Dmitry},
  booktitle={Proceedings of the 46th International ACM SIGIR Conference on Research and Development in Information Retrieval},
  pages={2067--2071},
  year={2023}
}

@misc{MagicPrompt,
    year = {2023},
    url = {https://github.com/XpycT/MagicPrompt},
    title = {MagicPrompt}
}

@misc{Lexica,
    year = {2022},
    url ={https://huggingface.co/datasets/Gustavosta/Stable-Diffusion-Prompts},
    title = {Lexica}
}

@inproceedings{clipbad2,
  title={CREPE: Can Vision-Language Foundation Models Reason Compositionally?},
  author={Ma, Zixian and Hong, Jerry and Gul, Mustafa Omer and Gandhi, Mona and Gao, Irena and Krishna, Ranjay},
  booktitle={Proceedings of the IEEE/CVF Conference on Computer Vision and Pattern Recognition},
  pages={10910--10921},
  year={2023}
}

@inproceedings{pr1,
title={Distill The Image to Nowhere: Inversion Knowledge Distillation for Multimodal Machine Translation},
author={Peng, Ru and Zeng, Yawen and Zhao, Jake},
booktitle={Proceedings of the 2022 Conference on Empirical Methods in Natural Language Processing},
pages={2379--2390},
year={2022}
}

@inproceedings{pr2,
title={HybridVocab: Towards Multi-Modal Machine Translation via Multi-Aspect Alignment},
author={Peng, Ru and Zeng, Yawen and Zhao, Junbo},
booktitle={Proceedings of the 2022 International Conference on Multimedia Retrieval},
pages={380--388},
year={2022}
}

@inproceedings{metric:imagereward,
  title={ImageReward: Learning and Evaluating Human Preferences for Text-to-Image Generation},
  author={Xu, Jiazheng and Liu, Xiao and Wu, Yuchen and Tong, Yuxuan and Li, Qinkai and Ding, Ming and Tang, Jie and Dong, Yuxiao},
  booktitle={Advances in Neural Information Processing Systems},
  volume={36},
  pages={15903--15935},
  year={2023}
}

@inproceedings{lewis-etal-2020-bart,
    title = "{BART}: Denoising Sequence-to-Sequence Pre-training for Natural Language Generation, Translation, and Comprehension",
    author = "Lewis, Mike  and
      Liu, Yinhan  and
      Goyal, Naman  and
      Ghazvininejad, Marjan  and
      Mohamed, Abdelrahman  and
      Levy, Omer  and
      Stoyanov, Veselin  and
      Zettlemoyer, Luke",
    editor = "Jurafsky, Dan  and
      Chai, Joyce  and
      Schluter, Natalie  and
      Tetreault, Joel",
    booktitle = "Proceedings of the 58th Annual Meeting of the Association for Computational Linguistics",
    month = jul,
    year = "2020",
    address = "Online",
    publisher = "Association for Computational Linguistics",
    url = "https://aclanthology.org/2020.acl-main.703",
    doi = "10.18653/v1/2020.acl-main.703",
    pages = "7871--7880",
}

@inproceedings{visualprompter,
  title={VisualPrompter: Semantic-Aware Prompt Optimization with Visual Feedback for Text-to-Image Synthesis},
  author={Wu, Shiyu and Sun, Mingzhen and Wang, Weining and Wang, Yequan and Liu, Jing},
  booktitle={International Conference on Learning Representations},
  year={2026}
}

@inproceedings{contextual_bandit,
  title={Contextual bandits with linear payoff functions},
  author={Chu, Wei and Li, Lihong and Reyzin, Lev and Schapire, Robert},
  booktitle={Proceedings of the fourteenth international conference on artificial intelligence and statistics},
  pages={208--214},
  year={2011},
  organization={JMLR Workshop and Conference Proceedings}
}

@misc{grpo,
      title={DeepSeekMath: Pushing the Limits of Mathematical Reasoning in Open Language Models}, 
      author={Zhihong Shao and Peiyi Wang and Qihao Zhu and Runxin Xu and Junxiao Song and Xiao Bi and Haowei Zhang and Mingchuan Zhang and Y. K. Li and Y. Wu and Daya Guo},
      year={2024},
      eprint={2402.03300},
      archivePrefix={arXiv},
      primaryClass={cs.CL},
      url={https://arxiv.org/abs/2402.03300}, 
}

@inproceedings{adamw,
    title={Decoupled Weight Decay Regularization},
    author={Ilya Loshchilov and Frank Hutter},
    booktitle={International Conference on Learning Representations},
    year={2019},
    url={https://openreview.net/forum?id=Bkg6RiCqY7},
}

@article{hpsv2,
  title={Human preference score v2: A solid benchmark for evaluating human preferences of text-to-image synthesis},
  author={Wu, Xiaoshi and Hao, Yiming and Sun, Keqiang and Chen, Yixiong and Zhu, Feng and Zhao, Rui and Li, Hongsheng},
  journal={arXiv preprint arXiv:2306.09341},
  year={2023}
}

@article{manasimproving,
  title={Improving Text-to-Image Consistency via Automatic Prompt Optimization},
  author={Ma{\~n}as, Oscar and Astolfi, Pietro and Hall, Melissa and Ross, Candace and Urbanek, Jack and Williams, Adina and Agrawal, Aishwarya and Romero-Soriano, Adriana and Drozdzal, Michal},
  journal={Transactions on Machine Learning Research},
  year={2024}
}

@article{scene-graph-survey,
  title={Scene graph generation: A comprehensive survey},
  author={Li, Hongsheng and Zhu, Guangming and Zhang, Liang and Jiang, Youliang and Dang, Yixuan and Hou, Haoran and Shen, Peiyi and Zhao, Xia and Shah, Syed Afaq Ali and Bennamoun, Mohammed},
  journal={Neurocomputing},
  volume={566},
  pages={127052},
  year={2024},
  publisher={Elsevier}
}

@article{visualgenome,
  title={Visual genome: Connecting language and vision using crowdsourced dense image annotations},
  author={Krishna, Ranjay and Zhu, Yuke and Groth, Oliver and Johnson, Justin and Hata, Kenji and Kravitz, Joshua and Chen, Stephanie and Kalantidis, Yannis and Li, Li-Jia and Shamma, David A and others},
  journal={International journal of computer vision},
  volume={123},
  number={1},
  pages={32--73},
  year={2017},
  publisher={Springer}
}

@inproceedings{impp,
  title={Scene graph generation by iterative message passing},
  author={Xu, Danfei and Zhu, Yuke and Choy, Christopher B and Fei-Fei, Li},
  booktitle={Proceedings of the IEEE conference on computer vision and pattern recognition},
  pages={5410--5419},
  year={2017}
}

@inproceedings{neuralmotifs,
  title={Neural motifs: Scene graph parsing with global context},
  author={Zellers, Rowan and Yatskar, Mark and Thomson, Sam and Choi, Yejin},
  booktitle={Proceedings of the IEEE conference on computer vision and pattern recognition},
  pages={5831--5840},
  year={2018}
}

@inproceedings{vctree,
  title={Learning to compose dynamic tree structures for visual contexts},
  author={Tang, Kaihua and Zhang, Hanwang and Wu, Baoyuan and Luo, Wenhan and Liu, Wei},
  booktitle={Proceedings of the IEEE/CVF conference on computer vision and pattern recognition},
  pages={6619--6628},
  year={2019}
}

@inproceedings{sgtr,
  title={Sgtr: End-to-end scene graph generation with transformer},
  author={Li, Rongjie and Zhang, Songyang and He, Xuming},
  booktitle={proceedings of the IEEE/CVF conference on computer vision and pattern recognition},
  pages={19486--19496},
  year={2022}
}

@inproceedings{dsg,
  author={Jaemin Cho and Yushi Hu and Jason M. Baldridge and Roopal Garg and Peter Anderson and Ranjay Krishna and Mohit Bansal and Jordi Pont-Tuset and Su Wang},
  title={Davidsonian Scene Graph: Improving Reliability in Fine-grained Evaluation for Text-to-Image Generation},
  year={2024},
  cdate={1704067200000},
  url={https://openreview.net/forum?id=ITq4ZRUT4a},
  booktitle={International Conference on Learning Representations}
}

@inproceedings{tifa,
  title={TIFA: Accurate and Interpretable Text-to-Image Faithfulness Evaluation with Question Answering},
  author={Hu, Yushi and Liu, Benlin and Kasai, Jungo and Wang, Yizhong and Ostendorf, Mari and Krishna, Ranjay and Smith, Noah A.},
  booktitle={Proceedings of the IEEE/CVF International Conference on Computer Vision},
  pages={24498--24508},
  year={2023}
}

@inproceedings{vqascore,
  title={Evaluating Text-to-Visual Generation with Image-to-Text Generation},
  author={Lin, Zhiqiu and Pathak, Deepak and Li, Baiqi and Li, Jiayao and Xia, Xide and Neubig, Graham and Zhang, Pengchuan and Ramanan, Deva},
  booktitle={European Conference on Computer Vision},
  pages={366--384},
  year={2024},
  organization={Springer}
}

@inproceedings{pae,
  title={Dynamic Prompt Optimizing for Text-to-Image Generation},
  author={Mo, Wenyi and Zhang, Tianyu and Bai, Yalong and Su, Bing and Wen, Ji-Rong and Yang, Qing},
  booktitle={Proceedings of the IEEE/CVF Conference on Computer Vision and Pattern Recognition},
  pages={26627--26636},
  year={2024}
}

@inproceedings{reprompt,
  title={RePrompt: Reasoning-Augmented Reprompting for Text-to-Image Generation via Reinforcement Learning},
  author={Wu, Mingrui and Wang, Lu and Zhao, Pu and Yang, Fangkai and Zhang, Jianjin and Liu, Jianfeng and Zhan, Yuefeng and Han, Weihao and Sun, Hao and Ji, Jiayi and Sun, Xiaoshuai and Lin, Qingwei and Deng, Weiwei and Zhang, Dongmei and Sun, Feng and Ji, Rongrong},
  booktitle={International Conference on Learning Representations},
  year={2026}
}

@inproceedings{dpok,
  title={DPOK: Reinforcement Learning for Fine-tuning Text-to-Image Diffusion Models},
  author={Fan, Ying and Watkins, Olivia and Du, Yuqing and Liu, Hao and Ryu, Moonkyung and Boutilier, Craig and Abbeel, Pieter and Ghavamzadeh, Mohammad and Lee, Kangwook and Lee, Kimin},
  booktitle={Advances in Neural Information Processing Systems},
  volume={36},
  pages={79858--79885},
  year={2023}
}

@article{attendandexcite,
  title={Attend-and-Excite: Attention-Based Semantic Guidance for Text-to-Image Diffusion Models},
  author={Chefer, Hila and Alaluf, Yuval and Vinker, Yael and Wolf, Lior and Cohen-Or, Daniel},
  journal={ACM Transactions on Graphics},
  volume={42},
  number={4},
  pages={148:1--148:10},
  doi={10.1145/3592116},
  year={2023}
}

@article{t2icompbench,
  title={T2I-CompBench: A Comprehensive Benchmark for Open-World Compositional Text-to-Image Generation},
  author={Huang, Kaiyi and Sun, Kaiyue and Xie, Enze and Li, Zhenguo and Liu, Xihui},
  journal={Advances in Neural Information Processing Systems},
  volume={36},
  pages={78723--78747},
  year={2023}
}

\clearpage
\appendix
\tableofcontents
\newpage

\section{Pipeline and Data Construction}
\label{appendix:1}
This section specifies how prompt pairs become supervised edit targets and how image feedback subsequently trains the policy. Appendices~\ref{app:theory} and~\ref{app:implementation} provide the reward analysis and implementation settings, respectively.

\subsection{Data Collection}
\label{appendix:data_collection}
To construct the SFT pairs, we start from each optimized prompt and derive two input variants: its semantic skeleton and a GPT-4o rewrite of that skeleton. Table~\ref{st_case} gives one example. The optimized-prompt source is the public Lexica-derived collection\footnote{\url{https://huggingface.co/datasets/Gustavosta/Stable-Diffusion-Prompts}}.

\begin{table}[htb]
\centering
        \caption{Example construction of an input--target pair from an optimized prompt.}
\resizebox{\linewidth}{!}{
        \begin{tabular}{lp{14cm}}
            \toprule
             \textbf{Optimized Prompt} & a fashion catwalk on a fishing boat, fashion photography, concept art, octane, redshift, high detail\\
             \midrule
            \textbf{Semantic Skeleton} & a fashion catwalk on a fishing boat\\
            \midrule
            \textbf{Rewriting} & on a fishing boat, there was a fashion catwalk \\
            \bottomrule
            \end{tabular}
            }
        
    \label{st_case}
\end{table}

\subsection{Supervised Fine-Tuning}
For supervised fine-tuning, we derive edit-action targets from each mapping $x\rightarrow y$ with a greedy token alignment that favors a long shared subsequence. The target action sequence is then constructed in three ordered steps:
\begin{enumerate}
    \item \textit{Deletion}: Iterate through the token list of the raw prompt $x$ and mark the tokens that do not appear in the alignment result to identify the deletion actions. 
    \item \textit{Reordering}: Record the mapping of tokens between their original positions and target positions during the alignment process to parse the reordering actions.
    \item \textit{Insertion}: Compare the token list of the optimized prompt to identify tokens that do not appear in the alignment result, and determine their target insertion positions to parse the insertion actions.
\end{enumerate}

\subsection{Edit-Action Illustration}
\label{app:edit_actions}
Figure~\ref{fig:edit_action_appendix} illustrates the ordered action space used by EAG. The editor first removes tokens that are not retained, then predicts successor links among the remaining tokens, and finally inserts new tokens into the resulting slots. The resulting trace links retained input tokens to their positions in the optimized prompt and records the inserted material.

\begin{figure}[htb]
\centering
\includegraphics[width=0.9\linewidth]{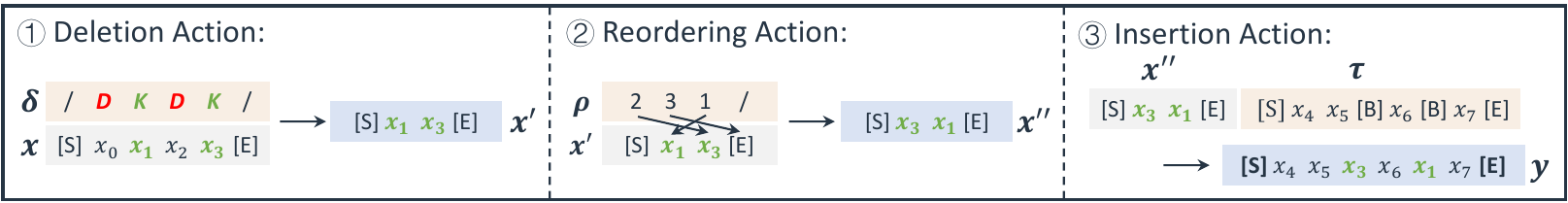}
\caption{Deletion, reordering, and insertion actions used by the edit-action generator.}
\label{fig:edit_action_appendix}
\end{figure}

\subsection{Reinforcement Learning}
\label{appendix:rl}
Algorithm~\ref{algorithm:rl} gives detailed pseudocode for environment-aware GRPO.
\begin{algorithm}[htb]
\caption{Environment-Aware GRPO for the Prompt Policy}
\label{algorithm:rl}
\begin{algorithmic}[1]
\Require SFT policy $\pi_{\theta}$, reference policy $\pi_{\mathrm{ref}}$, RL prompt pool $\mathcal{D}'$, frozen T2I model $g$, reward $r$.
\State \textbf{Hyperparameters:} updates $T$, minibatch size $B$, group size $N$, environment samples $M$, policy epochs $E$, clip $\epsilon$, KL weight $\beta$, entropy weight $\lambda$.
\For{$t=1$ to $T$}
    \State Sample a prompt minibatch $\mathcal{B}_x\subset\mathcal{D}'$ with $|\mathcal{B}_x|=B$; set $\theta_{\mathrm{old}}\leftarrow\theta$.
    \ForAll{$x\in\mathcal{B}_x$}
        \State Extract/calibrate $G_x$ and sample $N$ actions $a_{x,i}\sim\pi_{\theta_{\mathrm{old}}}(\cdot\mid x,G_x)$, yielding prompts $y_{x,i}$.
        \For{$i=1$ to $N$}
            \For{$j=1$ to $M$}
                \State Sample $I_{x,i,j}\sim g(y_{x,i})$ and compute $r_{x,i,j}=r(x,I_{x,i,j})$.
            \EndFor
            \State $\bar r_{x,i}\leftarrow M^{-1}\sum_{j=1}^{M}r_{x,i,j}$.
        \EndFor
        \State $\widehat A_{x,i}\leftarrow\dfrac{\bar r_{x,i}-\operatorname{mean}_{k}\bar r_{x,k}}{\operatorname{std}_{k}\bar r_{x,k}+\epsilon_A}$.
    \EndFor
    \For{$e=1$ to $E$}
        \State Compute $w_{x,i}=\pi_{\theta}(a_{x,i}\mid x,G_x)/\pi_{\theta_{\mathrm{old}}}(a_{x,i}\mid x,G_x)$ for all $x,i$.
        \State Update $\theta$ by ascending the minibatch-averaged clipped GRPO objective with KL and entropy regularization.
    \EndFor
\EndFor
\State \Return optimized policy $\pi_{\theta}$.
\end{algorithmic}
\end{algorithm}

\section{Theoretical Analysis}
\label{app:theory}

We provide the proof of Theorem~\ref{theorem:smooth_reward} here. For a fixed context $(x,G_x)$, let a rollout observation be $I=(a,i_y)$, where action $a$ deterministically produces prompt $y$ and $i_y\sim g(y)$. Define the grounding-confidence vector
\begin{align}
p(I) \coloneqq \big[P(z_1\mid I),\dots,P(z_K\mid I)\big]\in[0,1]^K,
\end{align}
where $\{z_k\}_{k=1}^K$ denotes the set of $K$ atomic semantic items (objects or relations) in the scene graph.
Define the discrete reward and the smooth reward, respectively, as
\begin{align}
&R_d(I) \coloneqq \frac{1}{K}\sum_{k=1}^K \mathbf{1}\!\big[P(z_k\mid I)\ge u\big],\\
&R_s(I) \coloneqq \frac{1}{K}\sum_{k=1}^K P(z_k\mid I),
\end{align}
where $u\in(0,1)$ is a fixed threshold. The single-sample policy gradient estimator is
\begin{align}
&\hat{\nabla}J(\theta;I) \coloneqq \Psi(\theta;I)\,R(I),
\\
&\Psi(\theta;I)\coloneqq \nabla_\theta \log \pi_\theta(a\mid x,G_x).
\end{align}
Throughout, $\|\cdot\|$ denotes the Euclidean norm and $\|\cdot\|_1$ denotes the $\ell_1$ norm.

\begin{theorem}
\label{thm:appendix_smooth_reward_proof}
Let $I$ be the observed information, and $p(I)=\big[P(z_1\mid I),\dots,P(z_K\mid I)\big]\in[0,1]^K$.
Define the discrete reward $R_d(I)=\frac{1}{K}\sum_{k=1}^K \mathbf{1}[P(z_k\mid I)\ge u]$
and the smooth reward $R_s(I)=\frac{1}{K}\sum_{k=1}^K P(z_k\mid I)$.
The policy gradient estimator is $\hat{\nabla}J(\theta;I)=\Psi(\theta;I)R(I)$,
where $\Psi(\theta;I)=\nabla_\theta\log\pi_\theta(a\mid x,G_x)$ for the action associated with rollout $I$.
When $R = R_s$, the gradient fluctuation admits the following upper bound:
\begin{align}
\label{eq:appendix_grad_fluctuation_bound}
\big\|\hat{\nabla}J(\theta; I) - \hat{\nabla}J(\theta; I')\big\|
\le
\big\|\Psi(\theta;I)-\Psi(\theta;I')\big\|
+\\
\big\|\Psi(\theta;I)\big\|\frac{1}{K}\big\|p(I)-p(I')\big\|_1.
\end{align}
In contrast, when $R = R_d$, such a Lipschitz-type bound generally does not hold.
\end{theorem}

\begin{proof}
\textbf{Step 1: A Lipschitz bound for the smooth reward.}
For any $I,I'$, we first show that
\begin{align}
\label{eq:appendix_Rs_lipschitz}
\big|R_s(I)-R_s(I')\big|
\le
\frac{1}{K}\big\|p(I)-p(I')\big\|_1.
\end{align}
Indeed, by definition,
\begin{align}
R_s(I)-R_s(I')
=
\frac{1}{K}\sum_{k=1}^K \Big(P(z_k\mid I)-P(z_k\mid I')\Big),
\end{align}
hence by the triangle inequality,
\begin{align}
\big|R_s(I)-R_s(I')\big|
&\le
\frac{1}{K}\sum_{k=1}^K \Big|P(z_k\mid I)-P(z_k\mid I')\Big|
\\&=
\frac{1}{K}\big\|p(I)-p(I')\big\|_1,
\end{align}
which proves~\eqref{eq:appendix_Rs_lipschitz}.

\textbf{Step 2: Decomposition of gradient difference.}
Let $R=R_s$ and abbreviate $\Psi(I)\coloneqq \Psi(\theta;I)$.
Then
\begin{align}
\hat{\nabla}J(\theta;I)-\hat{\nabla}J(\theta;I')
=
\Psi(I)R_s(I)-\Psi(I')R_s(I').
\end{align}
Add and subtract $\Psi(I)R_s(I')$ to obtain
\begin{align}
\label{eq:appendix_split}
\Psi(I)R_s(I)-\Psi(I')R_s(I')
=
\big(\Psi(I)-\Psi(I')\big)R_s(I')
\\+
\Psi(I)\big(R_s(I)-R_s(I')\big).
\end{align}
Taking norms and applying the triangle inequality yields
\begin{align}
\label{eq:appendix_triangle}
\big\|\hat{\nabla}J(\theta;I)-\hat{\nabla}J(\theta;I')\big\|
\le
\big\|\Psi(I)-\Psi(I')\big\|\,|R_s(I')|
\\+
\|\Psi(I)\|\,\big|R_s(I)-R_s(I')\big|.
\end{align}

\textbf{Step 3: Bounding each term.}
Since $P(z_k\mid I')\in[0,1]$ for all $k$, we have $R_s(I')\in[0,1]$, hence $|R_s(I')|\le 1$.
Combining this with~\eqref{eq:appendix_Rs_lipschitz} in~\eqref{eq:appendix_triangle}, we obtain
\begin{align}
\big\|\hat{\nabla}J(\theta;I)-\hat{\nabla}J(\theta;I')\big\|
\le
\big\|\Psi(I)-\Psi(I')\big\|
\\+
\|\Psi(I)\|\cdot \frac{1}{K}\big\|p(I)-p(I')\big\|_1,
\end{align}
which is exactly~\eqref{eq:appendix_grad_fluctuation_bound} and proves the first statement.

\textbf{Step 4: Failure of a Lipschitz-type bound for the discrete reward.}
We provide a rigorous counterexample to show that a uniform Lipschitz-type inequality generally cannot hold
for $R_d$.
Consider the case $K=1$, where $R_d(I)=\mathbf{1}[P(z_1\mid I)\ge u]$.
For any $0<\varepsilon<\min\{u,1-u\}$, choose two observations $I,I'$ such that
\begin{align}
P(z_1\mid I)=u-\varepsilon,
\\
P(z_1\mid I')=u+\varepsilon.
\end{align}
Then
\begin{align}
R_d(I)=0, R_d(I')=1,
\\\Rightarrow
|R_d(I)-R_d(I')|=1,
\end{align}
while
\begin{align}
\|p(I)-p(I')\|_1=\big|P(z_1\mid I)-P(z_1\mid I')\big|=2\varepsilon.
\end{align}
If there were a finite constant $L$ such that
$|R_d(I)-R_d(I')|\le L\|p(I)-p(I')\|_1$ held for all $I,I'$, we would have
\begin{align}
1 \le L\cdot 2\varepsilon
\Rightarrow
L\ge \frac{1}{2\varepsilon}.
\end{align}
Letting $\varepsilon\to 0^+$ forces $L\to+\infty$, contradicting the existence of a uniform finite $L$.
Therefore, such a Lipschitz-type bound generally does not hold for $R_d$. Holding a nonzero score function fixed transfers this discontinuity to the gradient contribution.
This completes the proof.
\end{proof}

\paragraph{Reward-variation control.}
The next statement bounds reward variation in terms of changes in grounding confidence. It complements the score-function decomposition above by isolating the reward-dependent part of the feedback.

\begin{theorem}
\label{thm:appendix_variance_control}
Let $X\coloneqq p(I)\in[0,1]^K$ be a random grounding-confidence vector, let $X'$ be an independent copy, and define $f(X)=K^{-1}\sum_{k=1}^K X_k$, so that $R_s(I)=f(p(I))$. Then
\begin{align}
\label{eq:appendix_var_Rs}
\mathrm{Var}\!\big(R_s(I)\big)
&=\frac{1}{2}\,\mathbb{E}\!\left[\big(f(X)-f(X')\big)^2\right] \\
&\le \frac{1}{2K^2}\,\mathbb{E}\!\left[\|X-X'\|_1^2\right].
\end{align}
By contrast, the thresholded reward may retain constant-order variance for confidence vectors concentrated arbitrarily close to the threshold.
\end{theorem}

\begin{proof}
For any square-integrable scalar random variable $Y$ and an independent copy $Y'$, $\mathrm{Var}(Y)=\frac12\mathbb{E}[(Y-Y')^2]$. Applying this identity to $Y=f(X)$ and using Equation~\eqref{eq:appendix_Rs_lipschitz} gives
\begin{align}
\mathrm{Var}(R_s(I))
&=\frac12\mathbb{E}\!\left[(f(X)-f(X'))^2\right]\\
&\le \frac{1}{2K^2}\mathbb{E}\!\left[\|X-X'\|_1^2\right].
\end{align}
For the contrast, take $K=1$, choose $0<\delta<\min\{u,1-u\}$, and let $X_1=u+\xi$ with $\xi\sim\mathrm{Unif}[-\delta,\delta]$. Then $R_d=\mathbf{1}[\xi\ge0]$ has variance $1/4$, whereas $R_s=X_1$ has variance $\delta^2/3\to0$ as $\delta\to0$. Thus, hard thresholding can amplify arbitrarily small confidence perturbations around $u$.
\end{proof}

\paragraph{Implication for the gradient second moment.}
If $\|\Psi(\theta;I)\|\le B$ almost surely, then for $G(I)=\Psi(\theta;I)R(I)$,
\begin{align}
\label{eq:appendix_grad_second_moment}
\mathbb{E}\!\left[\|G(I)\|^2\right]
\le B^2\mathbb{E}\!\left[R(I)^2\right].
\end{align}
Consequently, controlling the reward second moment controls an upper bound on the gradient second moment. Together with Theorem~\ref{thm:appendix_variance_control}, this explains why continuous confidence is less sensitive to small grounding perturbations than a hard indicator. This is a stability argument rather than a universal ordering of the complete gradient variance or a convergence theorem.

\paragraph{Clipped-log shaping.}
The implementation uses
\begin{align}
\phi(t)=\left[\log\!\big(\max\{t,\eta\}\big)\right]_{c_{\min}}^{c_{\max}},\qquad \eta>0,
\end{align}
where $c_{\min}\le c_{\max}$ are finite and scalar clipping is non-expansive. Since $t\mapsto\log(\max\{t,\eta\})$ is at most $1/\eta$-Lipschitz on $[0,1]$, $\phi$ is also Lipschitz on this domain. Therefore,
\begin{align}
\left|\frac1K\sum_k\phi(X_k)-\frac1K\sum_k\phi(X'_k)\right|
\le \frac{1}{K\eta}\|X-X'\|_1,
\end{align}
so the shaped reward remains continuous in the confidence vector. Writing $R_\phi(I)=K^{-1}\sum_k\phi(p_k(I))$ and $C_\phi=\max\{|c_{\min}|,|c_{\max}|\}$, the same decomposition gives
\[
\|\Psi(I)R_\phi(I)-\Psi(I')R_\phi(I')\|
\le C_\phi\|\Psi(I)-\Psi(I')\|
+\frac{\|\Psi(I)\|}{K\eta}\|p(I)-p(I')\|_1.
\]
Here the finite clipping bounds determine the coefficient of score-function variation; the coefficient need not equal the unit bound for $R_s$. The probability floor and clipping constants are fixed using the validation split; they are not selected on any test set.

\section{Implementation Details}
\label{app:implementation}

\subsection{Models and Frozen Evaluators}
Unless stated otherwise, the default generator is Stable Diffusion v1.5, object grounding uses GLIP-T, relation verification uses BLIP-base, and the prompt editor is initialized from BART-base. The frozen MHE encoder is OpenAI CLIP ViT-L/14 and the graph encoder is a two-layer GCN with hidden size 768. During training, the T2I generator, GLIP, BLIP, CLIP, and all reward models are frozen; only the GCN and prompt editor are updated.

The aesthetic term uses the released \texttt{christophschuhmann/improved-aesthetic-predictor} implementation\footnote{\url{https://github.com/christophschuhmann/improved-aesthetic-predictor}} with checkpoint \texttt{sac+logos+ava1-l14-linearMSE.pth}. It consumes a normalized 768-dimensional OpenAI CLIP ViT-L/14 image embedding and applies an MLP with widths $768\!\rightarrow\!1024\!\rightarrow\!128\!\rightarrow\!64\!\rightarrow\!16\!\rightarrow\!1$, totaling 927,969 parameters. The released model uses dropout rates 0.2/0.2/0.1 and was trained by mean-squared error on SAC+Logos+AVA1 ratings. The predictor is frozen in both training and evaluation.

\subsection{Optimization Hyperparameters}
All hyperparameters are selected once on a reserved Lexica validation split and then fixed for every test set and generation backbone. SFT uses the 144K paired examples for 56K updates (approximately 50 epochs). The RL pool contains 24K prompts and training uses 6K optimizer updates, with four policy rollouts and three independently sampled images per rollout for environment-aware averaging. Tables~\ref{tab:hyperparameter} and~\ref{tab:hyper_reward} list the optimization and reward settings, respectively.

\begin{table}[htb]
\centering
\caption{SFT and RL optimization settings.}
\label{tab:hyperparameter}
\begin{tabular}{lcc}
\toprule
\textbf{Hyperparameter} & \textbf{SFT} & \textbf{RL}\\
\midrule
Batch size & 128 & 16\\
Learning rate & $2\times10^{-5}$ & $2\times10^{-5}$\\
Optimizer updates & 56K & 6K\\
Optimizer & AdamW & AdamW\\
AdamW $\epsilon$ & $10^{-8}$ & $10^{-8}$\\
AdamW $(\beta_1,\beta_2)$ & $(0.9,0.999)$ & $(0.9,0.999)$\\
Weight decay & $10^{-2}$ & $10^{-2}$\\
Policy rollouts $N$ & -- & 4\\
Images per rollout $M$ & -- & 3\\
\bottomrule
\end{tabular}
\end{table}

\begin{table}[htb]
\centering
\caption{Reward and regularization weights used in RL.}
\label{tab:hyper_reward}
\begin{tabular}{lc}
\toprule
\textbf{Hyperparameter} & \textbf{Value}\\
\midrule
Structural reward weight $\alpha_1$ & 0.6\\
Aesthetic reward weight $\alpha_2$ & 1\\
Global CLIP weight $\alpha_3$ & 5\\
KL coefficient $\beta$ & 0.2\\
Entropy coefficient $\lambda$ & 0.05\\
\bottomrule
\end{tabular}
\end{table}

\subsection{Computational Budget}
Table~\ref{tab:comp_budget} summarizes the reported training runs on NVIDIA GeForce RTX 3090 GPUs with 24GB memory. SFT uses four GPUs for approximately 80 GPU-hours; RL uses five GPUs for approximately 100 GPU-hours. The DPOK comparison matches the 100-GPU-hour RL stage; \ours{} additionally incurs SFT cost.

\begin{table}[htb]
\centering
\caption{Compute for the reported StruPO training stages.}
\label{tab:comp_budget}
\begin{tabular}{lcc}
\toprule
& \textbf{SFT} & \textbf{RL}\\
\midrule
Number of GPUs & 4 & 5\\
Total GPU-hours & 80 & 100\\
\bottomrule
\end{tabular}
\end{table}

\section{Reproduction Details}

This section records the baseline configurations, human-study protocol, and statistical procedures used in the reported experiments.

\subsection{Baseline Setup}
\label{app:baseline_setup}

\paragraph{Template.}
The Template baseline constructs the optimized prompt by embedding the raw prompt into a pre-designed template. The specific template we used is:
\begin{mdframed}
    \textcolor{blue}{\textless Prompt\textgreater}, cinematic, colorful background, concept art, dramatic lighting, high detail, highly detailed, hyper realistic, intricate, intricate sharp details, octane render, smooth, studio lighting, trending on artstation
\end{mdframed}

\paragraph{Promptist.}
Promptist employs a GPT-2-based model with 117M parameters to rewrite the raw prompt. Unless otherwise specified, we follow the configuration and provided weights from the original paper for comparative experiments. The weights are trained through 15k steps of supervised fine-tuning with a batch size of 256 and further optimized using the PPO algorithm over 12k episodes with a batch size of 256.

\paragraph{MagicPrompt.}
MagicPrompt uses the released GPT-2 117M continuation checkpoint, trained for 150K steps on roughly 80K diffusion prompts~\cite{MagicPrompt}. Rather than generating an optimized prompt from scratch, it appends a generated continuation to the raw prompt. We use the official weights and decoding interface.

\paragraph{BeautifulPrompt.}
BeautifulPrompt uses the 1.1B-parameter BLOOM checkpoint with 24 transformer layers. The released model is supervised-fine-tuned for four epochs on 143K samples and then optimized for 5,000 PPO steps. We use the officially provided weights.

\paragraph{Human-crafted.}
Where a dataset provides user-refined target prompts, we treat those targets as the Human-crafted baseline. COCO does not provide corresponding optimized targets, so this baseline is not reported there.

\paragraph{Rephrase.}
We use GPT-3.5 and GPT-4o to rephrase each raw prompt with the following instruction:

\begin{mdframed}
Give a simple description of the image to generate a drawing prompt. \\
1. Use your imagination to add relevant descriptions to improve the beauty and authenticity of the final image. \\
2. Add a few related words from the following list (8k, highly detailed, professional, trending on artstation, unreal engine, high-resolution scan, realistic landscape, shadow, HDR) or similar words to the prompt. \\
3. Output a positive English prompt in 50 words or less. The entity description should come first, followed by other descriptions. \\
Input: \textcolor{blue}{\textless Prompt\textgreater} \\
Output:
\end{mdframed}

\paragraph{Few-shot.}
We additionally introduce a few-shot approach by prompting LLMs with 3 example input-output pairs. The examples are fixed, and the overall prompt template is:

\begin{mdframed}
    \small Give a simple description of the image to generate a drawing prompt. \\
    1. Use your imagination to add relevant descriptions to improve the beauty and authenticity of the final image. \\
    2. Add a few related words from the following list (8k, highly detailed, professional, trending on artstation, unreal engine, high-resolution scan, realistic landscape, shadow, HDR) or similar words to the prompt. \\
    3. Output a positive English prompt in 50 words or less. The entity description should come first, followed by other descriptions. \\
    Input: ilya kuvshinov has long sky blue hair \\
    Output: ilya kuvshinov with long sky blue hair, gold eyes, professional digital painting, concept art, unreal engine 5, 8k, cinematic, wlop, bubbles, tendrils in the background, art by greg rutkowski, pixiv art, junji ito, yoshitaka amano\\
    Input: the nature landscape was breathtakingly beautiful \\
    Output: nature landscape, aerial view, drone photography, cinematic, mountains and ocean, cinematic view, epic sky, detailed, concept art, high detail, warm lighting, volumetric, godrays, vivid, beautiful, trending on artstation, by jordan grimmer, art greg rutkowski \\
    Input: emma watson is portrayed as hermione granger, sitting next to a window and reading a book \\
    Output: portrait of emma watson as hermione granger sitting next to a window reading a book, focused expression, face focus, golden hour, art by Kenne Gregoire and krenz cushart, trending on artstation \\
    Input: \textcolor{blue}{\textless Prompt\textgreater}\\
    Output:
\end{mdframed}

\paragraph{VisualPrompter.}
We conduct our evaluation using the official code released by the authors on OpenReview\footnote{\url{https://openreview.net/forum?id=hIwVFRLaFy}}.
All hyperparameter settings follow the default configurations.
In addition, we use Qwen2.5-14B as the optimizer model and Qwen2.5-VL as the vision--language evaluator.

\paragraph{PAE.}
PAE~\cite{pae} jointly predicts prompt modifiers, token weights, and denoising-time ranges. Because its published interface differs from a text-only reprompter, we execute the learned token weights and time ranges through the official dynamic-conditioning pipeline while using SD v1.5 and the same Lexica inputs.

\paragraph{RePrompt.}
RePrompt~\cite{reprompt} is a learned, single-step GRPO reprompting policy trained from image-level preference, semantic, and composition rewards. We use the official implementation and render its text output with the same SD v1.5 generation and evaluation protocol used for the other prompt editors.

\paragraph{DPOK and Attend-and-Excite.}
These methods characterize alternative intervention points beyond reusable text. DPOK~\cite{dpok} updates diffusion parameters; our controlled run initializes from SD v1.5, uses the same 24K RL prompts and composite reward, and is capped at 100 RL GPU-hours. Attend-and-Excite~\cite{attendandexcite} intervenes in cross-attention/latent trajectories for every image and is evaluated with the official inference settings.

\paragraph{GPT-4o-SG-Iter.}
This high-budget test-time baseline alternates image generation, structural critique, and GPT-4o rewriting for three rounds. The full 2K-prompt, three-seed comparison and exact intervention accounting appear in Appendix~\ref{app:iterative_baseline}.

\subsection{Human Evaluation Setup}
\label{app:human_details}

\begin{figure*}[htb]
\centering
\includegraphics[width=0.94\linewidth]{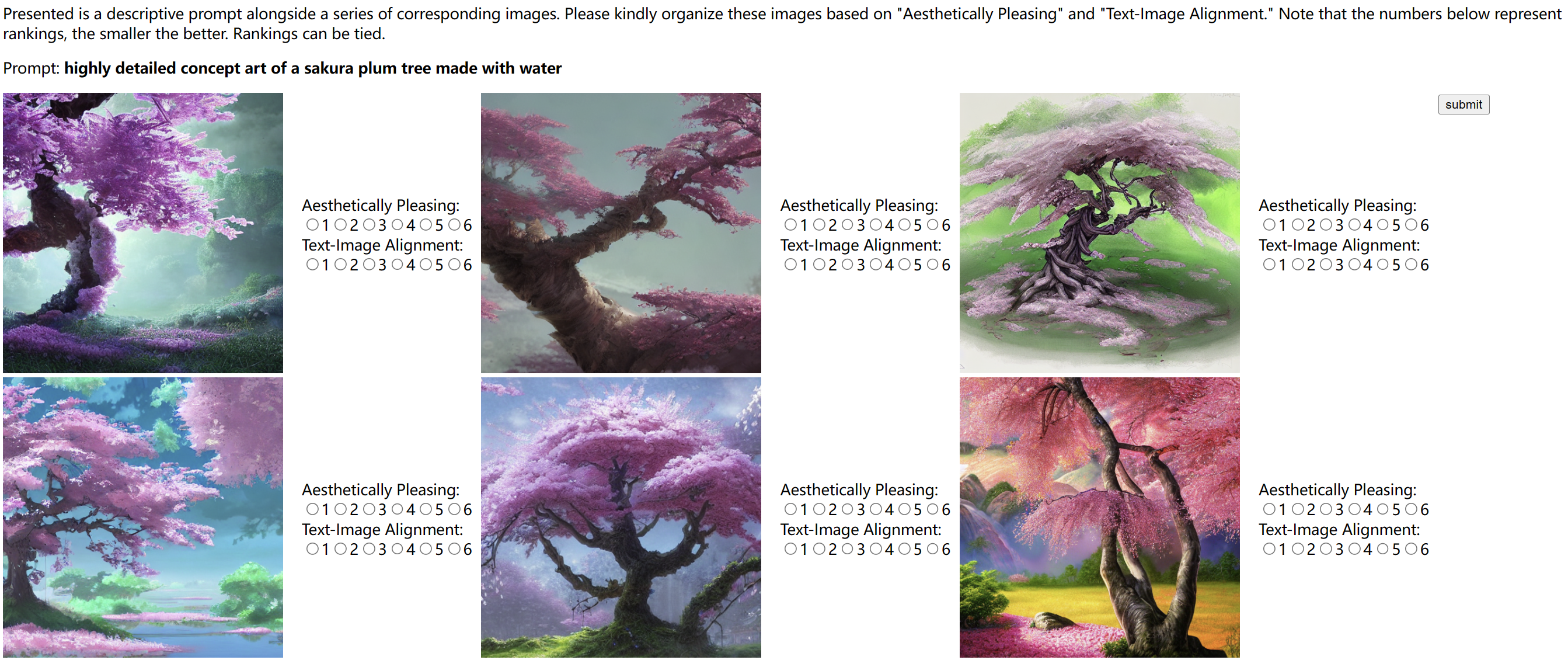}
\caption{Interface shown to annotators. Method names are hidden and left--right order is randomized.}
\label{interface}
\end{figure*}

\paragraph{Participants and task.}
Figure~\ref{interface} shows the anonymized comparison interface. Fifteen independent annotators, aged 20--30, participated. Their backgrounds include design, computer vision, media production, and visual-content evaluation. Personally identifying information was not included in the analysis. Annotators saw the original user prompt and two anonymized matched-seed images and judged (i) visual aesthetics and (ii) semantic consistency in separate questions. They could select the left image, the right image, or a tie. Tie responses were retained without applying a secondary criterion. Image order was independently randomized for every comparison, and annotators worked independently.

The same protocol was applied to all 2,000 prompts in each of Lexica, DiffusionDB, and COCO. Prompt-level win/tie/loss labels were obtained by majority aggregation across annotators. On Lexica, Fleiss' $\kappa$ is 0.40 for aesthetics and 0.43 for consistency; prompt-bootstrap 95\% half-widths are approximately $\pm2.1$ and $\pm2.4$ percentage points, respectively. Table~\ref{tab:human_all_datasets} summarizes the aggregate results on all three datasets.

\begin{table*}[htb]
\centering
\caption{Human preference for StruPO across all datasets. Each entry is StruPO win/tie/loss (\%) against the comparator.}
\label{tab:human_all_datasets}
\resizebox{0.94\textwidth}{!}{
\begin{tabular}{llcc}
\toprule
\textbf{Dataset} & \textbf{Comparator} & \textbf{Aesthetics W/T/L} & \textbf{Consistency W/T/L}\\
\midrule
Lexica & Promptist & 54/22/24 & 36/37/27\\
& VisualPrompter & 48/18/34 & 39/32/29\\
& Few-shot GPT-4o & 53/17/30 & 42/23/35\\
\midrule
DiffusionDB & Promptist & 60/19/21 & 44/33/23\\
& VisualPrompter & 67/16/17 & 62/21/17\\
& Few-shot GPT-4o & 59/20/21 & 43/33/24\\
\midrule
COCO & Promptist & 58/20/22 & 41/34/25\\
& VisualPrompter & 71/14/15 & 31/33/36\\
& Few-shot GPT-4o & 60/19/21 & 38/35/27\\
\bottomrule
\end{tabular}}
\end{table*}

StruPO has a positive W--L margin in 17 of the 18 dataset--comparator--criterion comparisons. On COCO, VisualPrompter is preferred for consistency (31/33/36 from StruPO's perspective), whereas StruPO is preferred for aesthetics (71/14/15). This matches the distinct quality--faithfulness operating points in the automatic evaluation.

\paragraph{Exact Lexica counts and significance.}
Table~\ref{tab:human_exact_counts} reports prompt-level counts for consistency. The 95\% intervals describe the W--L margin over all prompts; the $p$ values are two-sided exact sign tests over decisive (non-tie) prompts.

\begin{table}[htb]
\centering
\caption{Exact Lexica consistency results over 2,000 prompts.}
\label{tab:human_exact_counts}
\resizebox{0.96\linewidth}{!}{
\begin{tabular}{lcccc}
\toprule
\textbf{Comparator} & \textbf{W/T/L counts} & \textbf{Decisive win rate} & \textbf{95\% CI for W--L} & \textbf{Sign-test $p$}\\
\midrule
Promptist & 721/738/541 & 57.1\% & $[+5.5,+12.5]$ pp & $4.49\times10^{-7}$\\
VisualPrompter & 781/638/581 & 57.4\% & $[+6.4,+13.6]$ pp & $6.61\times10^{-8}$\\
Few-shot GPT-4o & 841/458/701 & 54.5\% & $[+3.2,+10.8]$ pp & $3.97\times10^{-4}$\\
\bottomrule
\end{tabular}}
\end{table}

\paragraph{Forced-choice sensitivity analysis.}
We repeated the evaluation on the same image pairs with a strict forced-choice instruction. A prompt is counted as a win only when a majority of annotators prefers StruPO. StruPO remains preferred under this protocol (Table~\ref{tab:human_forced_choice}), consistent with the direction of the tie-allowed results.

\begin{table}[htb]
\centering
\caption{Majority-aggregated StruPO win rates under strict forced choice on Lexica.}
\label{tab:human_forced_choice}
\begin{tabular}{lcc}
\toprule
\textbf{Comparator} & \textbf{Aesthetics} & \textbf{Consistency}\\
\midrule
Promptist & 65.2\% & 54.8\%\\
VisualPrompter & 57.0\% & 55.7\%\\
Few-shot GPT-4o & 61.7\% & 53.5\%\\
\bottomrule
\end{tabular}
\end{table}

\paragraph{Relation-complexity analysis.}
We reaggregate the original tie-allowed judgments using the relation-count buckets from the automatic complexity analysis: $|V_x|=1$ ($n=807$), $|V_x|=2$ ($n=694$), and $|V_x|\ge3$ ($n=499$). Table~\ref{tab:human_relation_complexity} reports W/T/L and the W--L margin in parentheses. Ordered multinomial-logistic trend tests remain significant after Holm correction.

\begin{table*}[htb]
\centering
\caption{Lexica consistency preference by relation complexity. Entries are W/T/L (W--L percentage points).}
\label{tab:human_relation_complexity}
\resizebox{0.94\textwidth}{!}{
\begin{tabular}{lccccc}
\toprule
\textbf{Comparator} & $|V_x|=1$ ($n=807$) & $|V_x|=2$ ($n=694$) & $|V_x|\ge3$ ($n=499$) & \textbf{Raw $p$} & \textbf{Holm $p$}\\
\midrule
Promptist & 34/38/28 (+6) & 35/38/27 (+8) & 41/34/25 (+16) & .0128 & .0256\\
VisualPrompter & 36/34/30 (+6) & 38/33/29 (+9) & 45/27/28 (+17) & .0064 & .0192\\
Few-shot GPT-4o & 40/23/37 (+3) & 41/24/35 (+6) & 47/20/33 (+14) & .0276 & .0276\\
\bottomrule
\end{tabular}}
\end{table*}

The human-perceptible consistency advantage therefore increases with structural complexity: the W--L margin rises by 10--11 points from the lowest to the highest relation bucket for all three strong baselines.

\paragraph{Consent, privacy, and compensation.}
Participants were informed about the task, data use, voluntary participation, and their right to withdraw; written consent was obtained. Identifiers were removed from the analysis and records were stored for research use. Participants were compensated at the study's standard hourly rate. The exact compensation rate and IRB/equivalent approval or exemption status are not specified in the study documentation.

\subsection{Error Bars and Statistical Stability}
\label{appendix:error_bars}
For automatic metrics, each prompt-level score is first averaged over three generated images. We then resample prompts with replacement for 1,000 bootstrap trials. Across Lexica, DiffusionDB, and COCO, the average 95\% half-widths are approximately $\pm0.02$ for AES$_\Delta$, $\pm0.6$ for OR, $\pm1.1$ for RR, and $\pm0.8$ for SGR. The pairwise ranking between StruPO and the strongest non-StruPO baseline remains unchanged in 91.2\% of bootstrap trials for SGR and 90.6\% for AES$_\Delta$. These intervals quantify prompt-sampling uncertainty; they do not account for uncertainty in the pretrained evaluators themselves.

\section{Additional Analyses}

We analyze sensitivity to prompt length and format, token-level attention, template composition, and interactive editing.

\begin{table*}[htb]
\centering
\caption{Results on different lengths of input prompts.}
\label{tab:length}
\resizebox{0.8\linewidth}{!}{
\begin{tabular}{lccccccccccccc}
\toprule
&\multicolumn{4}{c}{\textbf{0 $\sim$ 20}} &\multicolumn{4}{c}{{\textbf{20 $\sim$ 40}}}&\multicolumn{4}{c}{{\textbf{$\geq$ 40}}} \\
\cmidrule(lr){2-5}\cmidrule(lr){6-9}\cmidrule(lr){10-13}
\textbf{Models} & \textbf{AES}$_\Delta$ & \textbf{OR} & \textbf{RR} & \textbf{SGR} & \textbf{AES}$_\Delta$ & \textbf{OR} & \textbf{RR} & \textbf{SGR} & \textbf{AES}$_\Delta$ & \textbf{OR} & \textbf{RR} & \textbf{SGR} \\
\midrule
Template & 0.62 & 92.54 & 61.14 & 76.84 & 0.39 & 85.13 & 59.16 & 72.14 & 0.18 & 70.98 & 45.24 & 58.11 \\
Promptist & 0.59 & 92.53 & 59.86 & 76.20 & 0.30 & 85.50 & 63.85 & 74.68 & 0.09 & 66.73 & 34.52 & 50.63 \\
MagicPrompt & 0.56 & 91.96 & 60.63 & 76.29 & 0.33 & 85.95 & 62.71 & 74.33 & 0.14 & 73.70 & 44.05 & 58.87 \\
BeautifulPrompt & 0.68 & 90.34 & 46.20 & 68.27 & 0.40 & 78.80 & 35.99 & 57.39 & 0.25 & 63.53 & 14.29 & 38.91 \\
Human-crafted & 0.66 & 92.53 & 58.75 & 75.64 & 0.37 & 84.27 & 56.12 & 70.19 & 0.23 & 71.13 & 22.62 & 46.88 \\
Few-shot (GPT-4o) & 0.49 & 93.00 & 63.80 & 78.40 & 0.33 & 85.07 & 59.69 & 72.38 & 0.15 & 73.53 & 65.48 & 69.50 \\
\midrule
\rowcolor{mygray}\ours{} w/o RL (Ours) & 0.68 & 92.10 & 58.75 & 75.42 & 0.39 & 85.98 & 57.69 & 71.84 & \textbf{0.27} & 69.85 & \textbf{71.43} & \textbf{70.64} \\
\rowcolor{mygray}\ours{} (Ours)  & \textbf{0.78} & \textbf{94.04} & \textbf{71.97} & \textbf{83.01} & \textbf{0.43} & \textbf{88.70} & \textbf{68.13} & \textbf{78.41} & 0.16 & \textbf{76.65} & 61.90 & 69.28 \\
\bottomrule
\end{tabular}
}
\end{table*}

\subsection{Sensitivity to Prompt Settings}

We analyze the sensitivity of \ours{} to two prompt characteristics: input length and prompt format.

\paragraph{Prompt Length.}
Table~\ref{tab:length} stratifies results by raw-prompt length. Full \ours{} is best on all four metrics in the 0--20 and 20--40 token buckets. In the $\geq 40$ bucket, it retains the highest object recall (76.65), but the SFT-only variant (\ours{} w/o RL) has higher AES$_\Delta$, relation recall, and SGR (0.27/71.43/70.64 versus 0.16/61.90/69.28). The gains are therefore strongest for short and medium inputs. The long-input results identify a setting in which additional image-grounded training does not improve every objective over supervised editing.

\paragraph{Prompt Format.}
Table~\ref{tab:cat} separates rewritten inputs from semantic-skeleton inputs constructed as described in Appendix~\ref{appendix:1}. Within this Lexica evaluation, \ours{} has the highest AES$_\Delta$, OR, RR, and SGR in both format groups. Thus, the improvement holds for both input forms used to construct the prompt-optimization task.

\begin{table}[htb]
\centering
\caption{Results on different categories of input prompts.}
\label{tab:cat}
\resizebox{0.7\linewidth}{!}{
\begin{tabular}{lcccccccc}
\toprule
&\multicolumn{4}{c}{\textbf{Rewriting}} &\multicolumn{4}{c}{\textbf{Semantic Skeleton}} \\
\cmidrule(lr){2-5}\cmidrule(lr){6-9}
\textbf{Models} & \textbf{AES}$_\Delta$ & \textbf{OR} & \textbf{RR} & \textbf{SGR} & \textbf{AES}$_\Delta$ & \textbf{OR} & \textbf{RR} & \textbf{SGR} \\
\midrule
Template & 0.53 & 88.45 & 60.90 & 74.67 & 0.53 & 90.37 & 57.75 & 74.06 \\
Promptist & 0.46 & 88.39 & 63.44 & 75.92 & 0.49 & 90.51 & 57.99 & 74.25 \\
MagicPrompt & 0.46 & 89.35 & 63.14 & 76.24 & 0.47 & 89.46 & 58.46 & 73.96 \\
BeautifulPrompt & 0.58 & 83.86 & 40.95 & 62.41 & 0.55 & 87.34 & 38.22 & 62.78 \\
Human-crafted & 0.54 & 88.31 & 56.87 & 72.59 & 0.54 & 89.88 & 55.47 & 72.67 \\
Rephrase (GPT-3.5) & 0.43 & 88.18 & 62.55 & 75.37 & 0.42 & 91.25 & 60.36 & 75.81 \\
\midrule
\rowcolor{mygray}\ours{} w/o RL (Ours) & 0.59 & 88.37 & 60.24 & 74.30 & 0.55 & 90.47 & 56.06 & 73.27 \\
\rowcolor{mygray}\ours{} (Ours) & \textbf{0.63} & \textbf{91.25} & \textbf{70.45} & \textbf{80.85} & \textbf{0.64} & \textbf{92.23} & \textbf{68.44} & \textbf{80.33} \\
\bottomrule
\end{tabular}
}
\end{table}

\subsection{Token-Level Attention Visualization}
Following prior work, we visualize token-level diffusion attention with DAAM\footnote{\url{https://github.com/castorini/daam}} before and after prompt optimization. Figure~\ref{appendix:fig_show} shows randomly selected Lexica examples in which attention to scene-graph objects becomes more localized after editing. These qualitative maps illustrate token-level attention patterns alongside the quantitative grounding results.

\begin{figure}[htb]
\centering
\resizebox{0.5\linewidth}{!}{
\includegraphics{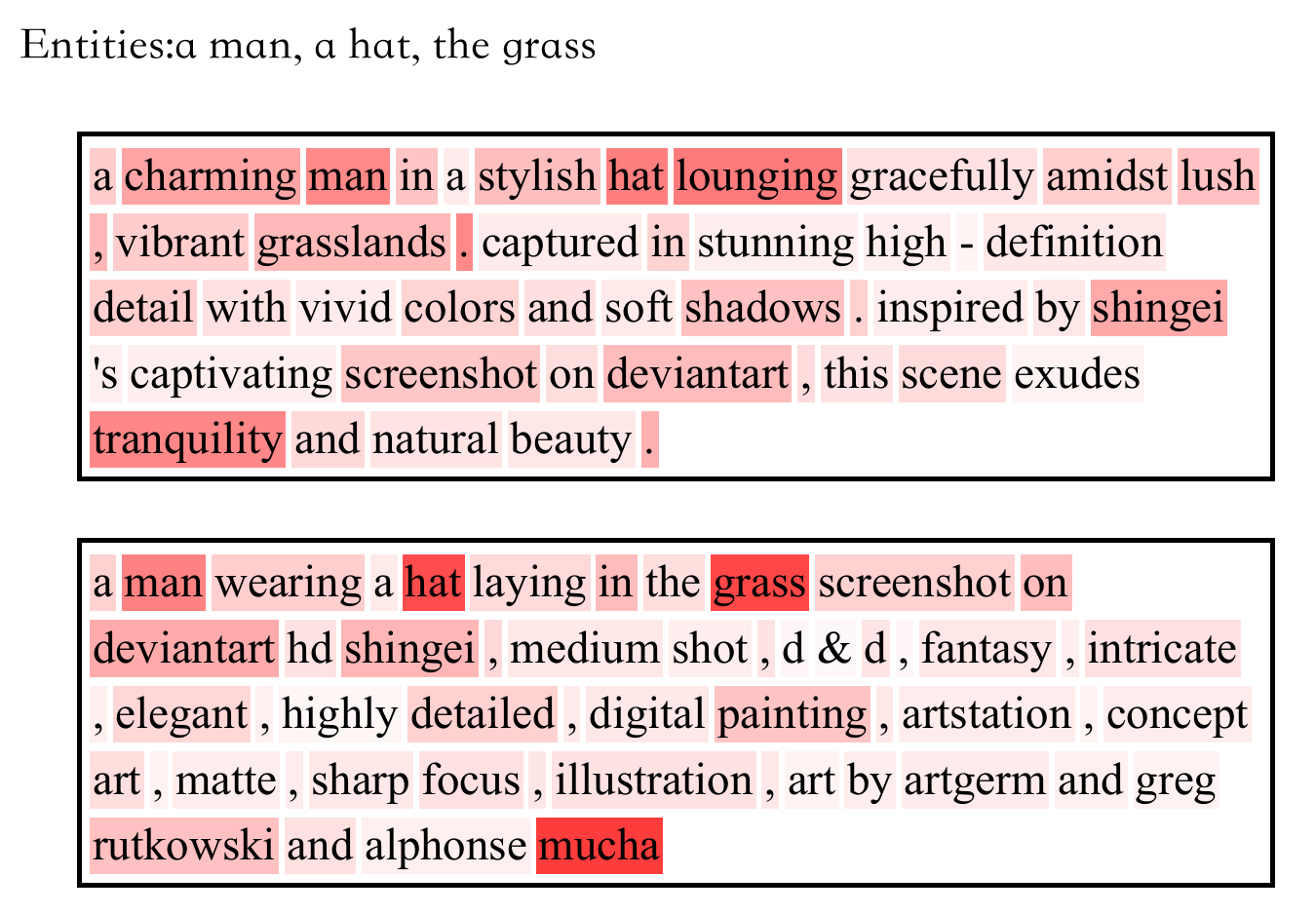}
}
\caption{DAAM attention visualization for a representative case: raw prompt (top) and \ours{}-optimized prompt (bottom).}
\label{appendix:fig_show}
\end{figure}

\subsection{Comparisons with Various Templates}

We test the style-suffix variants listed in Table~\ref{template}; Table~\ref{tab:template_result} reports the results. Some shorter variants improve SGR relative to the full suffix, but all reduce AES$_\Delta$. The strongest variant (Tag4) reaches SGR 76.36, still 4.38 points below \ours{}.

\begin{table}[htb]
\centering
\caption{Comparison with fixed style-suffix variants on Lexica.}
\label{tab:template_result}
\begin{tabular}{lcccc}
\toprule
\textbf{Method} & \textbf{AES}$_\Delta$ & \textbf{OR} & \textbf{RR} & \textbf{SGR} \\
\midrule
Full & 0.53 & 89.50 & 59.63 & 74.57 \\
Tag0 & 0.24 & 90.00 & 60.14 & 75.07 \\
Tag1 & 0.34 & 89.90 & 61.14 & 75.52 \\
Tag2 & 0.30 & 90.22 & 60.04 & 75.13 \\
Tag3 & 0.42 & 89.42 & 59.26 & 74.34 \\
Tag4 & 0.45 & 90.31 & 62.41 & 76.36 \\
Tag5 & 0.40 & 89.29 & 58.00 & 73.64 \\
Tag6 & 0.35 & 89.37 & 60.97 & 75.17 \\
Tag7 & 0.47 & 89.93 & 59.67 & 74.80 \\
\midrule
\rowcolor{mygray}\ours{} & \textbf{0.63} & \textbf{91.84} & \textbf{69.64} & \textbf{80.74} \\
\bottomrule
\end{tabular}
\end{table}

\begin{table*}[htb]
\centering
\caption{Style suffixes used in the template comparison.}
\label{template}
\resizebox{0.8\linewidth}{!}{
\begin{tabular}{cp{13cm}}
\toprule
\textbf{Tag} & \textbf{Content} \\
\midrule
Full& cinematic, colorful background, concept art, dramatic lighting, high detail, highly detailed, hyper realistic, intricate, intricate sharp details, octane render, smooth, studio lighting, trending on artstation \\
\midrule
0& dramatic lighting, smooth, hyper realistic, colorful background \\
\midrule
1& colorful background, concept art, octane render, hyper realistic\\
\midrule
2& smooth, highly detailed, octane render, intricate\\
\midrule
3& trending on artstation, concept art, high detail, dramatic lighting\\
\midrule
4& colorful background, octane render, dramatic lighting, studio lighting, high detail, smooth, concept art, intricate\\
\midrule
5& studio lighting, intricate sharp details, high detail, trending on artstation, dramatic lighting, octane render, highly detailed, hyper realistic\\
\midrule
6& colorful background, smooth, concept art, studio lighting, intricate, high detail, hyper realistic, dramatic lighting\\
\midrule
7& dramatic lighting, intricate sharp details, smooth, hyper realistic, octane render, colorful background, cinematic, high detail\\
\bottomrule
\end{tabular}
}
\end{table*}

\section{Supplementary Qualitative Material}

\subsection{Interactive Prompt Optimization}
Because \ours{} represents optimization as deletion, reordering, and insertion, the resulting edit trace can be inspected and constrained. A user interface can, for example, lock selected object phrases before decoding. The action representation exposes these controls directly to an interface.

Figure~\ref{appendix:fig_platform} shows a prototype front end for inspecting intermediate edits and specifying object-preservation constraints. The prototype illustrates an interaction design; no separate usability study is reported.

\begin{figure*}[htb]
\centering
\includegraphics[width=0.86\textwidth]{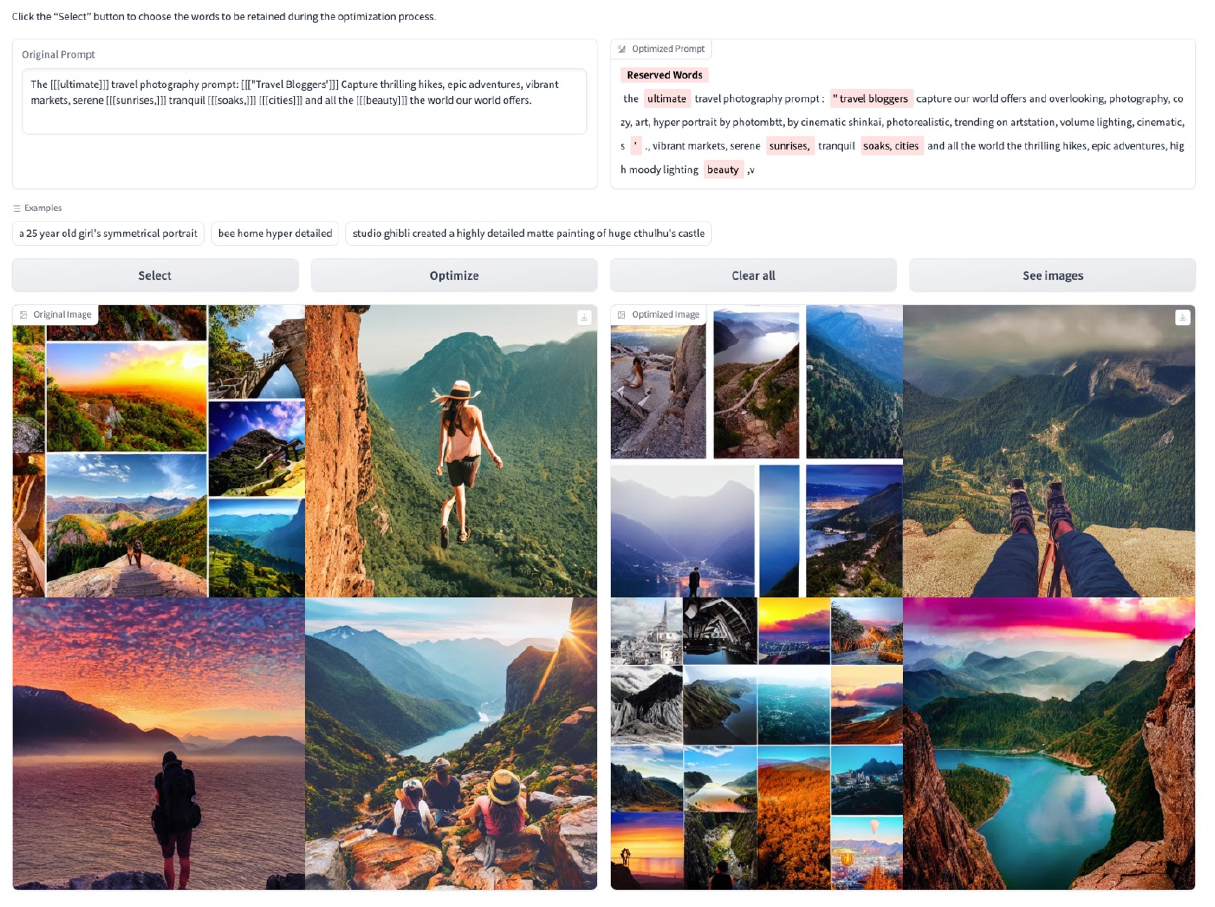}
\caption{Prototype interface for inspecting edits and specifying object-preservation constraints.}
\label{appendix:fig_platform}
\end{figure*}

\subsection{Case Study of Editing}
Figure~\ref{casestudy_edit} records intermediate deletion, reordering, and insertion outputs. The trace shows exactly where tokens are retained, moved, or introduced. It also exposes awkward word order and repeated modifiers, illustrating that an inspectable edit sequence need not be a fluent final description.

\begin{figure*}[htb]
\centering
\includegraphics[width=0.56\textwidth]{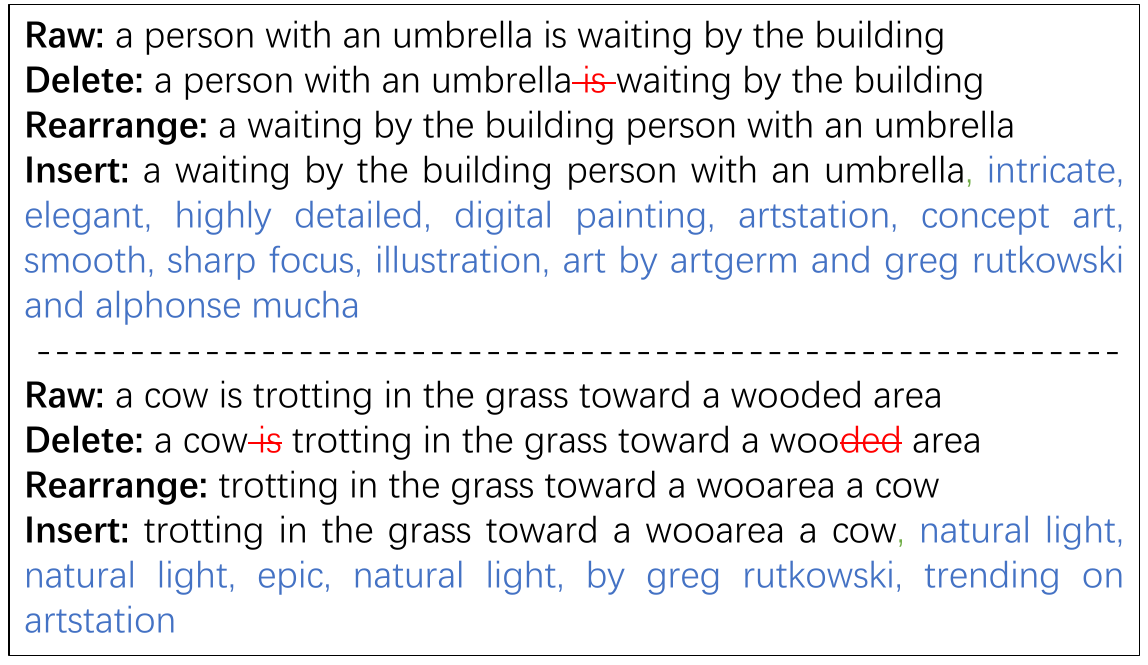}
\caption{Intermediate edit traces for two prompts, including residual word-order and repetition errors.}
\label{casestudy_edit}
\end{figure*}

\subsection{Additional Image Cases}
We provide additional visual examples generated by Stable Diffusion v1.5 (Figure~\ref{figure1}), Stable Diffusion v1.4 (Figure~\ref{figure2}), and Stable Diffusion v2.1 (Figure~\ref{figure3}).

\begin{figure*}[htb]
\centering
\includegraphics[width=0.86\textwidth]{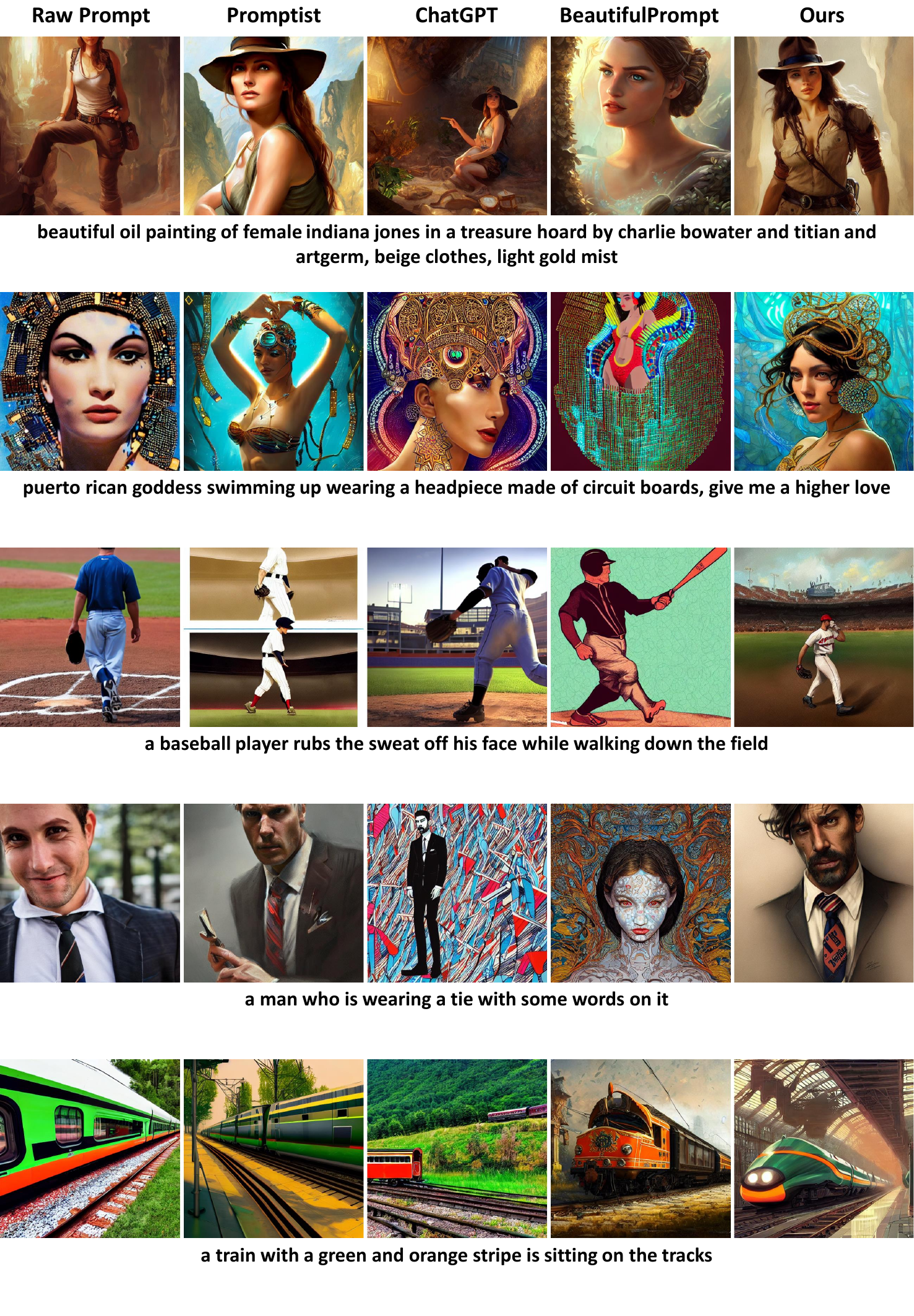}

\caption{More cases generated by Stable Diffusion v1.5.}
\label{figure1}
\end{figure*}
\clearpage

\begin{figure*}[htb]
\centering
\includegraphics[width=0.86\textwidth]{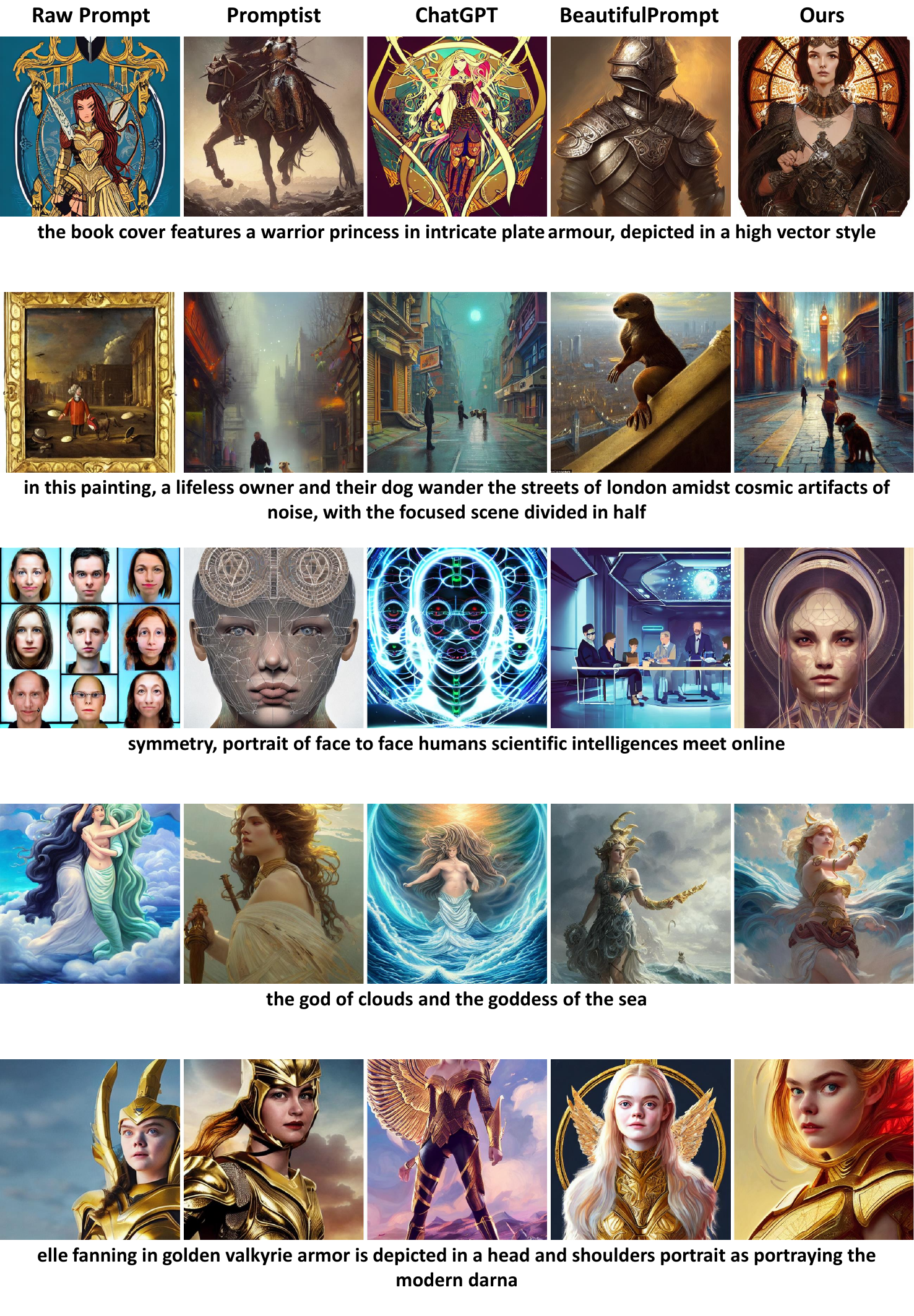}
\caption{More cases generated by Stable Diffusion v1.4.}
\label{figure2}
\end{figure*}
\clearpage

\begin{figure*}[htb]
\centering
\includegraphics[width=0.86\textwidth]{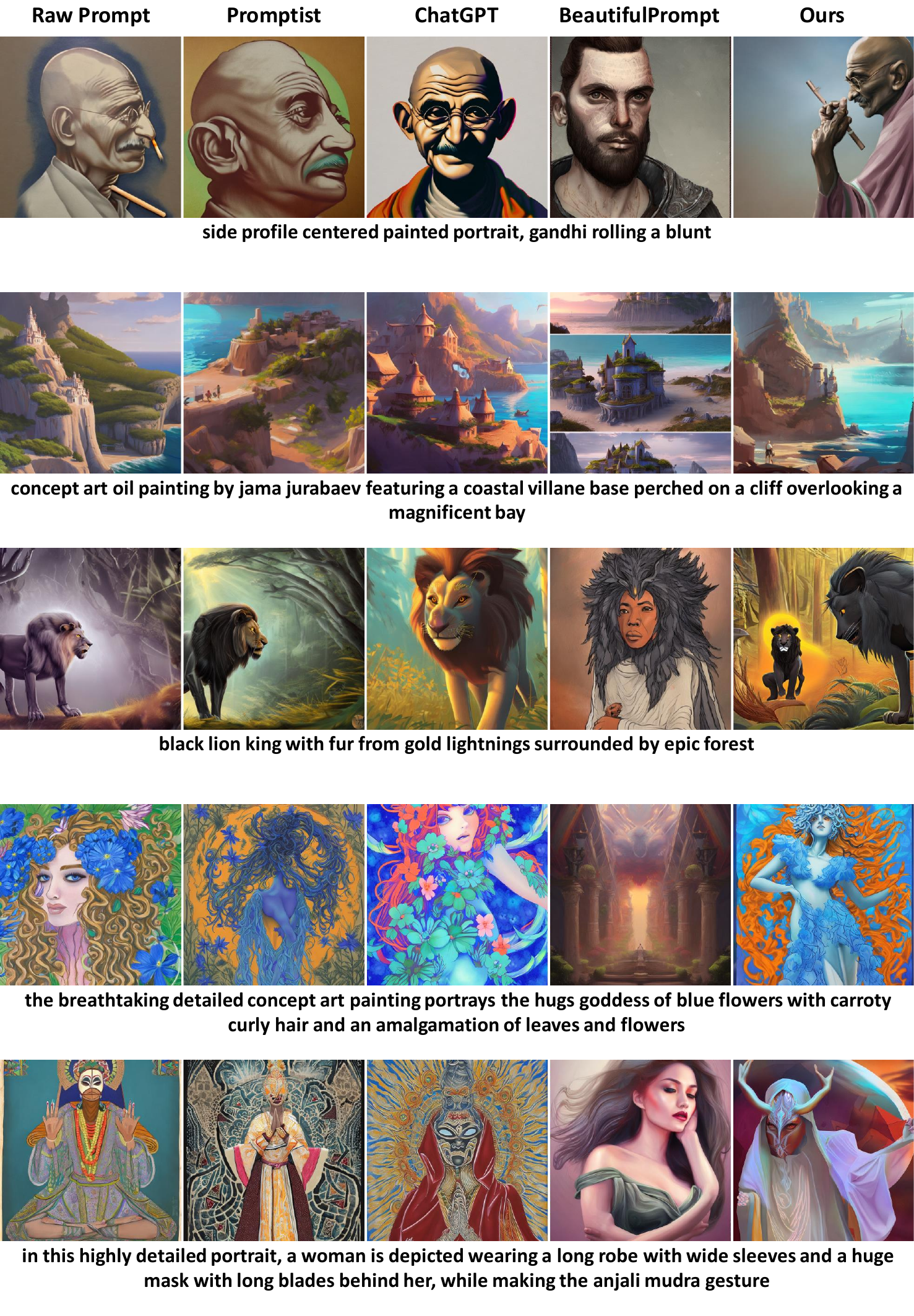}

\caption{More cases generated by Stable Diffusion v2.1.}
\label{figure3}
\end{figure*}

\section{Additional Evaluation and Robustness Analyses}
\label{app:additional_eval}
We organize the supplementary evidence around three questions: whether improvements transfer across evaluators, how the learned editor compares with alternative optimization strategies, and which structural and training choices drive its behavior. Human-study details and significance analyses appear in Appendix~\ref{app:human_details}.

\subsection{Gemini-VQA Protocol}
\label{app:gemini_vqa}
Gemini is used only as a held-out macro-level VQA judge. It receives the original prompt and one generated image, considers object presence, salient attributes, and main relations, and returns an overall prompt-satisfaction score from 0 to 100 plus a short explanation. Scores are averaged over the evaluation prompts. Gemini is not used in training, model selection, or reward construction. A representative instruction is:
\begin{quote}
\small
You are given a user prompt and a generated image. Evaluate how well the image satisfies the prompt at a coarse semantic level. Consider whether the key objects are present, whether salient attributes are correct, and whether the main relations are respected. Return an overall satisfaction score from 0 to 100, where 100 means the prompt is fully satisfied. Also provide a short explanation.
\end{quote}

\subsection{Independent Automatic Evaluators}
\label{app:external_metrics}
OR/RR/SGR measure the trained structural objective on held-out prompts. To evaluate the same images through separate scoring pipelines, we use three evaluators absent from default StruPO training: TIFA~\cite{tifa} with mPLUG-large, VQAScore~\cite{vqascore} with \texttt{clip-flant5-xxl}, and ImageReward-v1.0~\cite{metric:imagereward}. Every evaluator receives the original user prompt and the matched generated image. We average three images within each prompt and then all 2,000 prompts. Prompt IDs, T2I settings, and seed lists are shared across methods.

\begin{table*}[htb]
\centering
\caption{Independent evaluation on three disjoint 2K-prompt test sets. These evaluators are not used in the default StruPO training pipeline.}
\label{tab:independent_metrics_appendix}
\resizebox{0.94\textwidth}{!}{
\begin{tabular}{llccc}
\toprule
\textbf{Dataset} & \textbf{Method} & \textbf{TIFA} & \textbf{VQAScore} & \textbf{ImageReward}\\
\midrule
Lexica & Promptist & .736 & .711 & .751\\
& VisualPrompter & .713 & .705 & .717\\
& Few-shot GPT-4o & .730 & .716 & .779\\
& \textbf{StruPO} & \textbf{.789} & \textbf{.763} & \textbf{.804}\\
\midrule
DiffusionDB & Promptist & .748 & .724 & .728\\
& VisualPrompter & .563 & .541 & .492\\
& Few-shot GPT-4o & .755 & .731 & .757\\
& \textbf{StruPO} & \textbf{.793} & \textbf{.775} & \textbf{.807}\\
\midrule
COCO & Promptist & .726 & .706 & .675\\
& VisualPrompter & \textbf{.832} & \textbf{.808} & .496\\
& Few-shot GPT-4o & .768 & .741 & .698\\
& \textbf{StruPO} & .802 & .778 & \textbf{.732}\\
\bottomrule
\end{tabular}}
\end{table*}

Table~\ref{tab:independent_metrics_appendix} shows that StruPO leads all three independent metrics on Lexica and DiffusionDB. On COCO, VisualPrompter leads TIFA/VQAScore but has substantially lower ImageReward, consistent with its negative aesthetic change in the main table. The evaluator agreement on Lexica and DiffusionDB, together with the COCO trade-off, supports the quality--faithfulness pattern observed in the main results.

We additionally score the same 2,000 Lexica prompts with HPSv2.1~\cite{hpsv2}. HPSv2.1 places StruPO essentially level with few-shot GPT-4o, while ImageReward ranks StruPO first (Table~\ref{tab:hps_imagereward}). The additional coarse-metric evaluation reports gains of 0.21 PickScore points and 0.8 CLIPScore points over Promptist. PickScore summarizes preference alignment, while CLIPScore reflects global compatibility with a representation also used in auxiliary training terms. The relation-specific conclusions instead rely on grounding, separate VQA evaluation, and human judgments.

\begin{table}[htb]
\centering
\caption{Lexica preference evaluation. HPSv2.1 and ImageReward are independent of default StruPO training; the reward-aligned AES$_\Delta$ is included for comparison.}
\label{tab:hps_imagereward}
\begin{tabular}{lccc}
\toprule
\textbf{Method} & \textbf{AES$_\Delta$} & \textbf{HPSv2.1} & \textbf{ImageReward}\\
\midrule
Template & .53 & .262 & .628\\
Promptist & .48 & .271 & .751\\
VisualPrompter & .62 & .269 & .717\\
Few-shot GPT-4o & .51 & \textbf{.276} & .779\\
\textbf{StruPO} & \textbf{.63} & .275 & \textbf{.804}\\
\bottomrule
\end{tabular}
\end{table}

\subsection{Closest Learned Prompt Optimizers}
\label{app:learned_baselines}
PAE and RePrompt are evaluated with their official implementations under matched SD v1.5/Lexica inputs. For PAE, its learned token weights and denoising-time ranges are executed through the published dynamic-conditioning interface. RePrompt's text output is rendered with the same generator and matched evaluation setup. These comparisons assess complete optimization systems. The w/o-EAG ablation in Table~\ref{tab:ablation} separately isolates local editing while fixing graph conditioning, reward, and RL procedure.

\begin{table}[htb]
\centering
\caption{Recent learned prompt optimizers on matched SD v1.5/Lexica evaluation.}
\label{tab:learned_prompt_optimizers}
\resizebox{0.82\linewidth}{!}{
\begin{tabular}{lccccc}
\toprule
\textbf{Method} & \textbf{AES$_\Delta$} & \textbf{RR} & \textbf{SGR} & \textbf{TIFA} & \textbf{ImageReward}\\
\midrule
Promptist & .48 & 61.25 & 75.40 & .736 & .751\\
PAE & .58 & 62.71 & 76.42 & .756 & .786\\
RePrompt & .50 & 67.34 & 78.91 & \textbf{.795} & .799\\
\textbf{StruPO} & \textbf{.63} & \textbf{69.64} & \textbf{80.74} & .789 & \textbf{.804}\\
\bottomrule
\end{tabular}}
\end{table}

As reported in Table~\ref{tab:learned_prompt_optimizers}, StruPO improves over RePrompt by 2.30 RR and 1.83 SGR points. RePrompt obtains the higher TIFA score (.795 versus .789), while StruPO obtains higher AES$_\Delta$ and ImageReward. The comparison distinguishes relation-sensitive recall from broader faithfulness and preference scores.

\subsection{Attribute-Bearing Phrase Ablation}
\label{app:attribute_ablation}
DSE represents attributes within complete phrases such as ``black hair,'' ``gold crown,'' and ``blue king suit'' as object labels or relation arguments; MHE encodes those phrases and EAG performs token-aligned edits. To isolate this pathway, we use the color-, shape-, and texture/material-binding portions of T2I-CompBench~\cite{t2icompbench} with its disentangled BLIP-VQA evaluation. The policy-side ablation removes modifiers only from graph phrases serialized into MHE while retaining full reward queries. The reward-side ablation retains full policy inputs but removes modifiers from GLIP/BLIP grounding queries. Both are retrained with the same data, initialization, and budget.

\begin{table}[htb]
\centering
\caption{Attribute-binding diagnostic.}
\label{tab:attribute_phrase_ablation}
\resizebox{0.9\linewidth}{!}{
\begin{tabular}{lcccc}
\toprule
\textbf{Method / ablation} & \textbf{Color} & \textbf{Shape} & \textbf{Texture/material} & \textbf{Average}\\
\midrule
Promptist & .49 & .39 & .46 & .45\\
Policy-side head-noun graph & .54 & .42 & .51 & .49\\
Reward-side head-only grounding & .57 & .43 & .52 & .51\\
\textbf{Full attribute-bearing phrases} & \textbf{.60} & \textbf{.47} & \textbf{.56} & \textbf{.54}\\
\bottomrule
\end{tabular}}
\end{table}

Table~\ref{tab:attribute_phrase_ablation} shows that the full phrase representation gains 0.05 over the policy-side ablation and 0.03 over the reward-side ablation on average, with improvements in all three categories. Both interventions reduce binding accuracy, linking attribute preservation to the complete phrases supplied to the policy and to the grounding queries used in training.

\subsection{CLIP Sensitivity}
\label{app:clip_sensitivity}
The atomic structural reward itself is CLIP-free: GLIP scores objects and attribute-bearing phrases, and BLIP verifies relation-specific crops. CLIP appears only in auxiliary pathways: the frozen MHE encoder, the global regularizer $r_{\mathrm{CLIP}}$, and the image embedding consumed by the submitted aesthetic predictor. We isolate these roles with three variants trained using the same SFT data, BART initialization, 24K RL prompts, rollout/image budget, validation rule, and three training seeds. The T5 variant replaces only MHE's encoder with frozen T5-base plus a learned 768-dimensional projection. The CLIP-free stress test uses T5-base, removes $r_{\mathrm{CLIP}}$, and replaces the aesthetic training reward by ImageReward-v1.0.

\begin{table*}[htb]
\centering
\caption{CLIP sensitivity on Lexica. Each entry is the reported mean $\pm$ standard deviation over three training seeds. The CLIP-free variant replaces the aesthetic training reward with ImageReward.}
\label{tab:clip_sensitivity}
\resizebox{\linewidth}{!}{
\small
\setlength{\tabcolsep}{3.5pt}
\begin{tabular}{@{}lcccccc@{}}
\toprule
\textbf{Variant} & \textbf{AES$_\Delta$} & \textbf{OR} & \textbf{RR} & \textbf{SGR} & \textbf{TIFA} & \textbf{Gemini-VQA}\\
\midrule
Full StruPO & $.63\pm.014$ & $91.84\pm.17$ & $69.64\pm.31$ & $80.74\pm.21$ & $.789\pm.003$ & $82.4\pm.18$\\
Without $r_{\mathrm{CLIP}}$ & $.60\pm.018$ & $91.8\pm.21$ & $69.1\pm.43$ & $80.4\pm.24$ & $.785\pm.005$ & $82.1\pm.27$\\
T5 encoder + $r_{\mathrm{CLIP}}$ & $.64\pm.024$ & $91.9\pm.26$ & $69.5\pm.48$ & $80.6\pm.29$ & $.794\pm.006$ & $82.5\pm.31$\\
CLIP-free training & $.59\pm.033$ & $91.4\pm.34$ & $68.4\pm.62$ & $79.8\pm.39$ & $.781\pm.008$ & $81.7\pm.41$\\
\bottomrule
\end{tabular}}
\end{table*}

In Table~\ref{tab:clip_sensitivity}, replacing CLIP by T5 changes SGR by $-0.14$, within the reported run-to-run variation. Removing all CLIP-derived training components jointly lowers SGR by 0.94; the resulting 79.8 remains 4.40 points above Promptist. The structural advantage persists across encoder choices and after removing CLIP-derived training components, consistent with the central role of the GLIP/BLIP object--relation pathway.

\subsection{Cross-Paradigm Comparison}
\label{app:cross_paradigm}
Table~\ref{tab:cross_paradigm} compares three ways to intervene in T2I generation: diffusion-parameter optimization, attention-guided latent optimization, and a learned text editor. It reports image quality alongside the access requirements and per-image operations of each approach. The GPU-hour column isolates the matched RL budget; StruPO additionally uses the 80-GPU-hour SFT stage reported in Table~\ref{tab:comp_budget}.

\begin{table*}[htb]
\centering
\caption{Cross-paradigm comparison on 2,000 Lexica prompts. Access and cost (top) and image quality (bottom) describe the same runs. StruPO's 100-GPU-hour RL stage is preceded by 80 GPU-hours of SFT.}
\label{tab:cross_paradigm}
\small
\setlength{\tabcolsep}{4pt}
\resizebox{0.98\linewidth}{!}{
\begin{tabular}{@{}lccccc@{}}
\toprule
\textbf{Method} & \textbf{Optimized variable} & \shortstack{\textbf{Generator}\\\textbf{update}} & \shortstack{\textbf{Per-image}\\\textbf{optimization}} & \shortstack{\textbf{Reusable}\\\textbf{text}} & \shortstack{\textbf{Offline}\\\textbf{GPU-hours}}\\
\midrule
DPOK, matched reward & Diffusion parameters & Yes & No & No & 100 (RL)\\
Attend-and-Excite & Attention/latent trajectory & No & Yes & No & 0\\
\rowcolor{mygray}\textbf{StruPO} & Prompt policy & No & No & Yes & 100 (RL)\\
\bottomrule
\end{tabular}}
\par\medskip
\begin{tabular}{@{}lccccc@{}}
\toprule
\textbf{Method} & \textbf{AES$_\Delta$} & \textbf{OR} & \textbf{RR} & \textbf{SGR} & \textbf{TIFA}\\
\midrule
DPOK, matched reward & \textbf{.66} & 90.5 & 64.7 & 77.6 & .762\\
Attend-and-Excite & .31 & 91.7 & 66.3 & 79.1 & .787\\
\rowcolor{mygray}\textbf{StruPO} & .63 & \textbf{91.84} & \textbf{69.64} & \textbf{80.74} & \textbf{.789}\\
\bottomrule
\end{tabular}
\end{table*}

DPOK attains the largest aesthetic gain, while StruPO has the highest OR, RR, and SGR. StruPO and Attend-and-Excite have comparable TIFA scores (.789 and .787). StruPO improves RR by 3.34 points and SGR by 1.64 points, while producing reusable text without per-image latent optimization.

\subsection{Full GPT-4o Structural-Feedback Baseline}
\label{app:iterative_baseline}
We evaluate every method on the full 2,000-prompt Lexica test set and repeat the complete generation/evaluation pipeline with three independent seed lists. Prompt IDs, sampler, resolution, guidance scale, denoising steps, seed lists, and three-image averaging are matched across methods.

GPT-4o-SG-Iter first renders an image, obtains structural critique about missing objects, incorrect attributes, and broken relations, and asks GPT-4o to rewrite the current prompt. It performs three critique/rewrite rounds. Under a common raw-prompt-to-final-image boundary, StruPO uses two T2I renders, one DSE calibration (one GLIP pass plus relation-wise BLIP verification), and one policy pass; GPT-4o-SG-Iter uses four renders and three critique/rewrite rounds.

\begin{table*}[htb]
\centering
\caption{Full 2K-prompt Lexica comparison over three independent matched seed lists (mean $\pm$ standard deviation).}
\label{tab:iterative_baseline}
\resizebox{0.9\textwidth}{!}{
\begin{tabular}{lccccc}
\toprule
\textbf{Method} & \textbf{AES$_\Delta$} & \textbf{OR} & \textbf{RR} & \textbf{SGR} & \textbf{Gemini-VQA}\\
\midrule
Template & $.53\pm.01$ & $89.50\pm.14$ & $59.61\pm.28$ & $74.56\pm.17$ & $77.6\pm.2$\\
Promptist & $.48\pm.01$ & $89.57\pm.13$ & $61.20\pm.25$ & $75.39\pm.15$ & $79.5\pm.2$\\
VisualPrompter & $.62\pm.02$ & $90.21\pm.18$ & $53.55\pm.39$ & $71.88\pm.23$ & $79.2\pm.2$\\
GPT-4o-SG-Iter & $.61\pm.02$ & $91.08\pm.16$ & $66.55\pm.32$ & $78.82\pm.20$ & $81.6\pm.2$\\
\textbf{StruPO} & $\mathbf{.63\pm.02}$ & $\mathbf{91.82\pm.15}$ & $\mathbf{69.51\pm.29}$ & $\mathbf{80.67\pm.19}$ & $\mathbf{82.3\pm.2}$\\
\bottomrule
\end{tabular}}
\end{table*}

Table~\ref{tab:iterative_baseline} shows that, on the full test set, StruPO improves over GPT-4o-SG-Iter by 0.74 OR, 2.96 RR, 1.85 SGR, and 0.7 Gemini-VQA points. Thus, the learned editor combines higher structural recall with fewer test-time feedback rounds. The comparison characterizes the quality and online cost of reusable policy learning relative to per-prompt structural refinement.

\subsection{Reward Sensitivity}
\label{app:reward_sensitivity}
Table~\ref{tab:reward_sensitivity} varies the structural reward weight while leaving all other settings fixed. SGR stays between 79.6 and 80.7 across the tested weights, while AES$_\Delta$ decreases as the structural weight increases. SGR peaks at the intermediate setting, indicating that the trade-off is not monotonic. The default $\alpha_1=0.6$ is a balanced operating point selected on validation data.

\begin{table}[htb]
\centering
\caption{Sensitivity to structural reward weight $\alpha_1$ on Lexica.}
\label{tab:reward_sensitivity}
\begin{tabular}{ccc}
\toprule
$\alpha_1$ & \textbf{SGR} & \textbf{AES$_\Delta$}\\
\midrule
0.4 & 79.6 & .69\\
0.6 & \textbf{80.7} & .63\\
0.8 & 80.4 & .61\\
\bottomrule
\end{tabular}
\end{table}

\subsection{DSE Stress Test}
\label{app:dse_stress}
We inject controlled parser noise before DSE. Mild noise inserts distractor objects and swaps predicates; medium noise increases the perturbation rate and also drops relation endpoints. The full model under medium noise remains above the text-only policy, whereas using the raw parsed graph without DSE performs worse (Table~\ref{tab:dse_robustness}). Within these synthetic perturbations, calibration retains a benefit even as graph quality deteriorates.

\begin{table}[htb]
\centering
\caption{DSE robustness under synthetic scene-graph perturbations on Lexica.}
\label{tab:dse_robustness}
\begin{tabular}{lcc}
\toprule
\textbf{Setting} & \textbf{SGR} & \textbf{Remark}\\
\midrule
Clean calibrated SG & \textbf{80.7} & default\\
Mild SG noise & 79.8 & distractors / predicate swaps\\
Medium SG noise & 77.9 & stronger perturbation\\
Text-only policy & 75.3 & without graph input\\
Without DSE & 73.7 & raw parsed SG\\
\bottomrule
\end{tabular}
\end{table}

\subsection{Complexity-Bucket Protocol}
\label{app:automatic_complexity}
The automatic analyses in Figure~\ref{fig:constraint_complexity}\subref{fig:object_complexity} and~\subref{fig:relation_complexity} partition the 2,000 Lexica prompts by the number of retained object phrases ($|O_x|=1$, $|O_x|=2$, and $|O_x|\ge3$) and relation triplets ($|V_x|=1$, $|V_x|=2$, and $|V_x|\ge3$), respectively. Every method is evaluated on the same prompts and matched image seeds within each bucket. The human reaggregation in Table~\ref{tab:human_relation_complexity} uses the identical relation-count partition, allowing the automatic and perceptual trends to be compared without changing the prompt subsets.

\subsection{Runtime and Intervention Accounting}
\label{app:runtime}
On the reported SD v1.5 setup, policy decoding takes approximately 0.4 seconds and the original end-to-end implementation reports approximately 2.2 seconds per prompt. These timings refer to the reported hardware and implementation. Under the raw-prompt-to-final-image boundary, StruPO performs one policy pass after calibration; the matched GPT-4o baseline requires three multimodal critique/rewrite rounds and two additional T2I renders under the common boundary described above.

\subsection{Dense and Abstract Failure Cases}
\label{app:failure_cases}
On dense scenes with at least four objects, SGR drops from 80.7 to 74.8, but remains above Promptist (68.9) and few-shot GPT-4o (70.6). On abstract prompts, the consistency margin is smaller because the scene graph under-represents emotions, metaphors, and non-visual concepts. These cases identify two priorities for extension: reasoning over denser scenes and representing intent beyond visually grounded object--relation constraints.

\section{Structural Feedback in Learned Prompt Editing}
\label{appendix:combination_novelty}
Structural feedback can enter prompt optimization in different places: it can score candidates in an external search loop, inform a learned policy, or supervise that policy's image outcomes. Table~\ref{tab:method_scope} locates the compared methods along these axes. \ours{} couples the latter two roles through one calibrated graph and a local edit representation. The graph supplies the semantic anchors available to the editor, and its grounded elements determine which rendered outcomes receive higher reward.

\begin{table*}[htb]
\centering
\caption{Where structural feedback enters prompt optimization. ``Test-time only'' denotes graph input to an external feedback loop. The final column refers to an explicit delete/reorder/insert policy.}
\label{tab:method_scope}
\small
\setlength{\tabcolsep}{5pt}
\resizebox{0.98\linewidth}{!}{
\begin{tabular}{@{}lccccc@{}}
\toprule
\textbf{Method} & \shortstack{\textbf{Learned}\\\textbf{policy}} & \shortstack{\textbf{Calibrated}\\\textbf{SG input}} & \shortstack{\textbf{Atomic SG}\\\textbf{training reward}} & \shortstack{\textbf{Test-time}\\\textbf{feedback loop}} & \shortstack{\textbf{Local edit}\\\textbf{actions}}\\
\midrule
Promptist~\cite{promptist} & Yes & No & No & No & No\\
BeautifulPrompt~\cite{beautifulprompt} & Yes & No & No & No & No\\
RePrompt~\cite{reprompt} & Yes & No & No & No & No\\
PAE~\cite{pae} & Yes & No & No & No & No\\
OPT2I~\cite{manasimproving} & No & No & No & Yes & No\\
VisualPrompter~\cite{visualprompter} & No & Test-time only & No & Yes & No\\
\rowcolor{mygray}\textbf{StruPO} & \textbf{Yes} & \textbf{Yes} & \textbf{Yes} & \textbf{No} & \textbf{Yes}\\
\bottomrule
\end{tabular}}
\end{table*}

The components serve complementary roles in this learning problem. DSE calibrates the structural context; MHE exposes it to the editor; EAG links candidate changes to input tokens; and confidence-preserving, environment-aware optimization handles noisy image feedback. Table~\ref{tab:ablation} tests these roles through separate interventions. In particular, replacing local edits with free-form generation retains the graph, reward, and RL procedure but lowers SGR by 7.06 points. Conversely, retaining the editing architecture without structural reward lowers SGR by 5.60 points. The evidence therefore supports both the editing parameterization and the supervision it receives.

This design also yields a practical distinction from iterative feedback: after calibration, a single policy pass returns text that can be reused across generations. The full comparison in Table~\ref{tab:iterative_baseline} measures that online trade-off, while Table~\ref{tab:exp_generalize} evaluates reuse across frozen backbones. The contribution is a learned connection between structured intent and image-effective edits, with amortized inference as a consequence of that formulation.

\paragraph{Broader impact.}
\label{app:broader_impact}
Faithful prompt editing can reduce manual prompt-engineering effort and unintended changes to a user's requested scene. More controllable, visually convincing generation can also facilitate misleading synthetic content or reproduce biases in the underlying generator and feedback models. Deployment should retain the generator's safety filters and usage restrictions, communicate the synthetic origin of outputs, and respect privacy and copyright constraints on prompts, training assets, and generated content. The editor improves how intent is expressed to a generator; it does not itself establish that a requested use is appropriate.

\section{Declaration of LLM and Multimodal-Model Usage}
\label{app:llm-usage}
\paragraph{LLM and Multimodal Models Used in Experiments.}
We use several pretrained language and multimodal models as components of
experimental pipelines. Specifically, GPT-4o is used as an external
structural-feedback baseline (GPT-4o-SG-Iter) and Gemini-VQA is used as an
independent evaluation model for assessing generated images.
These models are used only for benchmark comparison and evaluation purposes;
they are not involved in designing the StruPO framework, defining the
optimization objective, or generating the proposed method.

\paragraph{LLM Assistance During Research and Manuscript Preparation.}
During the research and writing process, we used GPT-6 Pro as an auxiliary
tool for manuscript polishing, grammar refinement, and code assistance.
GPT-6 Pro was not used for proposing research ideas, developing the method,
designing experiments, interpreting experimental results, or making
scientific conclusions. All research directions, algorithmic designs,
experimental protocols, and analyses were developed and verified by the
authors.

\newpage
\section*{NeurIPS Paper Checklist}

\begin{enumerate}

\item {\bf Claims}
    \item[] Question: Do the main claims made in the abstract and introduction accurately reflect the paper's contributions and scope?
    \item[] Answer: \answerYes{}
    \item[] Justification: The abstract and introduction state the frozen-generator prompt-optimization scope, frame the contribution as sharing calibrated structural intent between policy conditioning and image-grounded supervision, and summarize the controlled ablations, independent evaluation, and observed trade-offs.

\item {\bf Limitations}
    \item[] Question: Does the paper discuss the limitations of the work performed by the authors?
    \item[] Answer: \answerYes{}
    \item[] Justification: The Limitations paragraph discusses dependence on parsers, grounding/VQA models and generators, incomplete coverage of abstract concepts, additional rollout compute, edit-space rigidity, and limited language/model coverage; Appendix~\ref{app:failure_cases} gives concrete failure cases.

\item {\bf Theory assumptions and proofs}
    \item[] Question: For each theoretical result, does the paper provide the full set of assumptions and a complete (and correct) proof?
    \item[] Answer: \answerYes{}
    \item[] Justification: The assumptions and scope of Theorem~\ref{theorem:smooth_reward} are stated in the main text, and Appendix~\ref{app:theory} provides the proof and explicitly limits the result to continuity and reward-variation control rather than convergence or universal variance reduction.

    \item {\bf Experimental result reproducibility}
    \item[] Question: Does the paper fully disclose all the information needed to reproduce the main experimental results of the paper to the extent that it affects the main claims and/or conclusions of the paper (regardless of whether the code and data are provided or not)?
    \item[] Answer: \answerNo{}
    \item[] Justification: The paper reports architecture, datasets and splits, model versions, baseline protocols, evaluator settings, optimization counts, matched seeds, and compute. However, the supplied revision does not expose the numerical clipped-log floor/clipping constants or a verified executable code/checkpoint package, so full reproduction cannot yet be claimed.

\item {\bf Open access to data and code}
    \item[] Question: Does the paper provide open access to the data and code, with sufficient instructions to faithfully reproduce the main experimental results, as described in supplemental material?
    \item[] Answer: \answerNo{}
    \item[] Justification: Public datasets and pretrained models are cited, but the materials supplied for this revision do not contain a verified anonymized implementation/checkpoint package with complete run commands, environment specification, and licensing documentation.

\item {\bf Experimental setting/details}
    \item[] Question: Does the paper specify all the training and test details (e.g., data splits, hyperparameters, how they were chosen, type of optimizer) necessary to understand the results?
    \item[] Answer: \answerYes{}
    \item[] Justification: Section~4 and Appendices~\ref{app:implementation}--\ref{app:baseline_setup} specify the disjoint splits, model components, optimizer settings, rollout/image counts, validation-only selection rule, baseline interfaces, and evaluator protocols needed to interpret the results.

\item {\bf Experiment statistical significance}
    \item[] Question: Does the paper report error bars suitably and correctly defined or other appropriate information about the statistical significance of the experiments?
    \item[] Answer: \answerYes{}
    \item[] Justification: Appendix~\ref{appendix:error_bars} reports prompt-bootstrap uncertainty; Appendix~\ref{app:human_details} reports exact counts, confidence intervals, two-sided sign tests, forced-choice replication, and Holm-corrected complexity trends; the full iterative baseline reports mean and standard deviation over three seed lists.

\item {\bf Experiments compute resources}
    \item[] Question: For each experiment, does the paper provide sufficient information on the computer resources (type of compute workers, memory, time of execution) needed to reproduce the experiments?
    \item[] Answer: \answerNo{}
    \item[] Justification: Appendix~\ref{app:implementation} reports GPU type, memory, device counts, and GPU-hours for the final SFT/RL runs and the matched DPOK budget. Total compute for preliminary, failed, or otherwise unreported development runs is not available in the supplied records.
    
\item {\bf Code of ethics}
    \item[] Question: Does the research conducted in the paper conform, in every respect, with the NeurIPS Code of Ethics \url{https://neurips.cc/public/EthicsGuidelines}?
    \item[] Answer: \answerNo{}
    \item[] Justification: The manuscript documents consent, privacy, limitations, and broader-impact considerations, but the supplied materials do not establish the human-study IRB/equivalent approval or exemption status or the exact compensation rate needed to verify every administrative requirement.

\item {\bf Broader impacts}
    \item[] Question: Does the paper discuss both potential positive societal impacts and negative societal impacts of the work performed?
    \item[] Answer: \answerYes{}
    \item[] Justification: The Broader Impact paragraph in Appendix~\ref{app:broader_impact} discusses reduced user effort and semantic drift as benefits, alongside misleading or biased synthetic content, copyright, privacy, synthetic-content provenance, and generator-level safeguards.
    
\item {\bf Safeguards}
    \item[] Question: Does the paper describe safeguards that have been put in place for responsible release of data or models that have a high risk for misuse (e.g., pre-trained language models, image generators, or scraped datasets)?
    \item[] Answer: \answerNA{}
    \item[] Justification: The supplied revision does not release an image generator or a scraped image dataset. It recommends preserving the underlying generator's safety filters and usage policies for any future prompt-editor deployment.

\item {\bf Licenses for existing assets}
    \item[] Question: Are the creators or original owners of assets (e.g., code, data, models), used in the paper, properly credited and are the license and terms of use explicitly mentioned and properly respected?
    \item[] Answer: \answerNo{}
    \item[] Justification: The paper cites the datasets, pretrained models, and implementations and identifies key versions/checkpoints, but the supplied materials do not exhaustively enumerate the license and terms of use for every asset and the Lexica-derived data source.

\item {\bf New assets}
    \item[] Question: Are new assets introduced in the paper well documented and is the documentation provided alongside the assets?
    \item[] Answer: \answerNo{}
    \item[] Justification: The method produces a learned prompt editor and derived training/evaluation artifacts, but a verified anonymized package containing model cards, data documentation, licenses, limitations, and exact reproduction commands is not present in the supplied materials.

\item {\bf Crowdsourcing and research with human subjects}
    \item[] Question: For crowdsourcing experiments and research with human subjects, does the paper include the full text of instructions given to participants and screenshots, if applicable, as well as details about compensation (if any)? 
    \item[] Answer: \answerNo{}
    \item[] Justification: Appendix~\ref{app:human_details} provides the task instructions, protocol, screenshot, consent/privacy procedures, and states that compensation was hourly, but the exact rate and jurisdictional minimum-wage comparison are not available in the supplied records.

\item {\bf Institutional review board (IRB) approvals or equivalent for research with human subjects}
    \item[] Question: Does the paper describe potential risks incurred by study participants, whether such risks were disclosed to the subjects, and whether Institutional Review Board (IRB) approvals (or an equivalent approval/review based on the requirements of your country or institution) were obtained?
    \item[] Answer: \answerNo{}
    \item[] Justification: The appendix describes a low-risk image-preference task, informed consent, withdrawal rights, privacy, and data handling. No IRB/equivalent approval or exemption statement is contained in the supplied materials, so approval is not claimed.

\item {\bf Declaration of LLM usage}
    \item[] Question: Does the paper describe the usage of LLMs if it is an important, original, or non-standard component of the core methods in this research? Note that if the LLM is used only for writing, editing, or formatting purposes and does \emph{not} impact the core methodology, scientific rigor, or originality of the research, declaration is not required.
    \item[] Answer: \answerYes{}
    \item[] Justification: Appendix~\ref{app:llm-usage} discloses GPT-4o use in SFT-input construction and comparison methods, GPT-3.5 use as a baseline, Gemini use as a held-out evaluator, and editorial assistance, while distinguishing these roles from the GLIP/BLIP structural reward.

\end{enumerate}
\end{document}